\documentclass[final, 12pt]{colt2019arx}

\usepackage{cmap}
\usepackage[T1]{fontenc}
\usepackage{bm}
\usepackage{amsmath}
\usepackage{amsfonts}
\usepackage{amssymb}
\usepackage{amsbsy}
\usepackage{dsfont}

\usepackage{mathtools}

\usepackage{enumitem}
\usepackage{hyperref}
\usepackage{multirow}
\usepackage{array}

\usepackage[outline]{contour}\usepackage{xcolor}

\usepackage[linewidth=2pt,
linecolor=gray,
middlelinecolor= black,
middlelinewidth=0.4pt,
roundcorner=1pt,
topline = false,
rightline = false,
bottomline = false,
rightmargin=0pt,
skipabove=0pt,
skipbelow=0pt,
leftmargin=0pt,
innerleftmargin=4pt,
innerrightmargin=0pt,
innertopmargin=0pt,
innerbottommargin=0pt,
]{mdframed}

\usepackage{chngcntr}
\usepackage{soul}
\usepackage{arydshln}

\newcommand{\N}{\ensuremath{\mathbb{N}}}

\newcommand{\R}{\ensuremath{\mathbb{R}}}

 \renewcommand{\vec}[1]{\boldsymbol{#1}}
\newcommand{\matr}[1]{\boldsymbol{#1}}

\newcommand{\matB}{\ensuremath{\boldsymbol{B}}}

\newcommand{\matT}{\ensuremath{\boldsymbol{T}}}

\newcommand{\vecs}{\ensuremath{\boldsymbol{s}}}
\newcommand{\vect}{\ensuremath{\boldsymbol{t}}}

\newcommand{\vecx}{\ensuremath{\boldsymbol{x}}}
\newcommand{\vecy}{\ensuremath{\boldsymbol{y}}}
\newcommand{\vecz}{\ensuremath{\boldsymbol{z}}}

  \newtheorem{inftheorem}{Informal Theorem}
 \newtheorem{claim}[theorem]{Claim}

\numberwithin{equation}{section}

\newenvironment{prevproof}[2]{\noindent {\bf {Proof of
{#1}~\ref{#2}:}}}{$\blacksquare$\vskip \belowdisplayskip}

\contourlength{0.1pt}
\contournumber{10}

\newif\ifnotes\notestrue

\ifnotes
\usepackage{color}
\definecolor{mygrey}{gray}{0.50}
\newcommand{\notename}[2]{{\textcolor{red}{\footnotesize{\bf (#1:} {#2}{\bf
) }}}}

\else

\newcommand{\notename}[2]{{}}

\fi

\newcommand{\geometric}{\mathcal{G}}

\DeclareMathOperator*{\Exp}{\mathbb{E}}
\DeclareMathOperator*{\Prob}{\mathbb{P}}

\DeclareMathOperator*{\Var}{\mathrm{Var}}
\DeclareMathOperator*{\KL}{\mathrm{D}_{\mathrm{KL}}}
\DeclareMathOperator*{\TV}{\mathrm{d}_{\mathrm{TV}}}

\mathchardef\mdash="2D
\newcommand{\eps}{\varepsilon}

\renewcommand{\epsilon}{\varepsilon}

\def\compactify{\itemsep=0pt \topsep=0pt \partopsep=0pt \parsep=0pt}
\let\latexusecounter=\usecounter

\newenvironment{Enumerate}
  {\def\usecounter{\compactify\latexusecounter}
   \begin{enumerate}}
  {\end{enumerate}\let\usecounter=\latexusecounter}

\def\eps{\varepsilon}

\DeclareMathOperator*{\argmin}{argmin}

\def\floor#1{\mathop{\left\lfloor#1\right\rfloor}}

\def\abs#1{\left|#1\right|}
\def\p#1{\left(#1\right)}
\def\b#1{\left[#1\right]}
\def\set#1{\left\{#1\right\}}

\newcommand{\paragr}[1]{\noindent \textbf{#1}}

\def\norm#1{\left\|#1\right\|}

\newcommand{\reals}{\mathbb{R}}

\newcommand{\chara}{\mathds{1}}
 
\newcommand{\expFamily}{\mathcal{E}}
\newcommand{\distr}{\mathcal{P}}
\newcommand{\Mallows}{\mathcal{M}}
\newcommand{\tgeometric}{\mathcal{TG}}
\newcommand{\family}{\mathcal{F}}
\newcommand{\risk}{\mathcal{R}}
\newcommand{\support}{\mathcal{X}}
\newcommand{\Domain}{\mathcal{D}}

\newcommand{\veta}{\vec{\eta}}
\newcommand{\sign}{\mathrm{sign}}

\title[Optimal Learning Mallows Model]{Optimal Learning for Mallows Block Model}
\usepackage{times}
\coltauthor{ \Name{R\'obert Busa-Fekete} \Email{busafekete@verizonmedia.com}\\
 \addr Yahoo! Research
 \AND
 \Name{Dimitris Fotakis} \Email{fotakis@cs.ntua.gr}\\
 \addr National Technical University of Athens
 \AND
 \Name{Bal\'{a}zs Sz\"{o}r\'enyi} \Email{balazs.szorenyi@verizonmedia.com}\\
 \addr Yahoo! Research
 \AND
 \Name{Manolis Zampetakis} \Email{mzampet@mit.edu}\\
 \addr Massachusetts Institute of Technology}

\begin{document}

\maketitle

\begin{abstract}
  The Mallows model, introduced in the seminal paper of \cite{Mal57}, is one of the most fundamental ranking distribution over the symmetric group $S_m$. To analyze more complex ranking data, several studies considered the Generalized Mallows model \citep{FlignerV1986,DoPeRe04,Mar95}. Despite the significant research interest of ranking distributions, the exact sample complexity of estimating the parameters of a Mallows and a Generalized Mallows Model is not well-understood.

  The main result of the paper is a tight sample complexity bound for learning Mallows and Generalized Mallows Model. We approach the learning problem by analyzing a more general model which interpolates between the single parameter Mallows Model and the $m$ parameter Mallows model. We call our model \textit{Mallows Block Model} -- referring to the Block Models that are a popular model in theoretical statistics. Our sample complexity analysis gives tight bound for learning the Mallows Block Model for any number of blocks. We provide essentially matching lower bounds for our sample complexity results.

  As a corollary of our analysis, it turns out that, if the central ranking is known, one single sample from the Mallows Block Model is sufficient to estimate the spread parameters with error that goes to zero as the size of the permutations goes to infinity. In addition, we calculate the exact rate of the parameter estimation error.

\end{abstract}


  \section{Introduction} \label{sec:intro}

The Mallows model is one of the most fundamental ranking distribution since it was introduced in the seminal paper of \cite{Mal57}. The model has two parameters, the \emph{central ranking} $\pi_0 \in S_m$ and the \emph{spread parameter} $\phi\in [0,1]$. Based on these, the probability of observing a ranking $\pi \in S_m$ is proportional to $\phi^{d ( \pi, \pi_0 )}$, where $d$ is a ranking distance, such as the number of discordant pairs, a.k.a Kendall's tau distance.

To capture more complicated distributions over rankings, several studies considered the generalized Mallows model \citep{FlignerV1986,DoPeRe04,Mar95}, which assigns a different spread parameter $\phi_i \in [0, 1]$ to each alternative $i$. Now the probability of observing $\pi \in S_m$ decreases exponentially in a weighted sum over the discordant pairs, where the weights are determined by the spread parameters of discordant items. Statistical estimation of the distribution and the parameters of the Mallows model has been of interest in a wide range of scientific areas including theoretical statistics \citep{Mukherjee16}, machine learning \citep{LuBo11,AwasthiBS14,ChenBBK09,MeilaB10}, social choice \citep{CaragiannisPS16}, theoretical computer science \citep{LiuM18} and many more, as we discuss in Section~\ref{sec:relatedWork}.

Despite this extensive literature, to the best of our knowledge, no optimal results are known on the sample complexity of learning the parameters of a Mallows or a generalized Mallows model. In this work, we fill this gap by proving: (1) an upper bound on the number of samples needed by some simple estimators to accurately estimate the parameters of the Mallows model, (2) an essentially matching lower bound on the sample complexity of any accurate estimator. Using our tight sample analysis, we are able to quantify in the finite sample regime some results that were only known in the asymptotic regime (e.g., \cite{Mukherjee16}).

Additionally, we introduce the \emph{Mallows Block model}, which interpolates between the simple Mallows and the generalized Mallows models. The definition of the Mallows Block model is similar in spirit to the (fundamental in theoretical statistics) Stochastic Block model \citep{KloppTV17}, which admits similar statistical properties. Also, \cite{BerthetRS16} recently introduced the Ising Block model, which is conceptually similar to the Stochastic Block Model. As we prove, the Mallows Block model combines two nice properties: (a) like the generalized Mallows model, it describes a wider range of distributions over rankings than the Mallows model; and (b) it allows accurate estimation of the spread parameters even from one sample, as it has been proved in \citep{Mukherjee16} for the Mallows model. We analyze the sample complexity of the Mallows Block model by proving essentially tight upper and lower bounds when the block structure is known.

\subsection{Results and Techniques} \label{sec:contrib}

In this work, we fully determine the sample complexity of learning Mallows and Generalized Mallows distributions, in a unified way, via the definition of the Mallows Block model. In a nutshell, we show how to estimate the parameters of these distributions in a (sample and time) efficient way, and how this implies efficient density estimation in KL-divergence and in total variation distance. Our approach is general and exploits properties of the exponential family. As we illustrate in Section~\ref{sec:concentration}, the use of these properties might useful in proving the exact learning rates for other complicated exponential families, such as the Ising model.
\smallskip

\paragr{Learning in KL-divergence.} Our learning algorithm for the spread parameters essentially finds the maximum likelihood solution, but in a provably computationally efficient way. The sample complexity analysis of the consistency of our estimator is based on some known and some novel results about exponential families.
As we see in Theorem~\ref{thm:exponentialFamiliesProperties}.4, the KL-divergence of two distributions in an exponential family is equal to the square difference of their parameters multiplied by the variance of a corresponding distribution inside the exponential family. If we put this together with Theorem~\ref{thm:exponentialFamiliesConcentration}, where we obtain a new strong concentration inequality for distributions in an exponential family, we get a systematic way of proving upper bounds on the number of samples required to learn an exponential family in KL-divergence. Thus, we depart from the (only known) upper bounds on density estimation in total variation distance. We apply our technique to the Mallows Block model and get tight upper bounds of  $O\p{\frac{d}{\eps^2} + \log\p{m}}$ samples, where $d$ is the (known) number of blocks in the Mallows Block model. We sketch the statement of this result below, for a formal statement see Theorem~\ref{thm:learningKLDivergence:blockModel}.

\begin{inftheorem}    Given $n = \tilde{\Omega} \p{\frac{d}{\eps^2} + \log\p{m}}$ samples from a Mallows $d$-Block
  distribution $\distr$, we can learn a distribution
  $\hat{\distr}$ such that $\KL\p{\distr||\hat{\distr}} \le \eps^2$ and hence
  $\TV\p{\distr, \hat{\distr}} \le \eps$.
\end{inftheorem}

\paragr{Parameter Estimation.} Extending a result of \cite{CaragiannisPS16}, we show that a logarithmic number of samples is both sufficient and necessary to estimate the central ranking of a generalized Mallows distribution (Theorem~\ref{thm:centralRankingLearningBlockModel}). Then, using our results on exponential families, we show that estimating the spread parameter $\phi$ of a Mallows distribution boils down to obtaining a lower bound on the KL-divergence between two Mallows distributions with the same central ranking and parameters $|\phi - \phi'| = \Theta(\eps)$.
With such a lower bound on the KL-divergence, we can apply the concentration inequality of Theorem~\ref{thm:exponentialFamiliesConcentration}, and show that once we learn the central ranking, with additional $O\left(\frac{d}{m^{\star} \eps^2}\right)$ i.i.d. samples, we can estimate the parameter vector $\vec{\phi}$ of the underlying Mallows Block model within $\ell_2$ error at most $\eps$. Here, $d$ denotes the number of blocks of the Mallows Block model and $m^{\star}$ is the minimum size of any block. We put everything together in the following informal theorem and refer to Theorem~\ref{thm:MallowsBlockParameterLearning} for a formal statement.

\begin{inftheorem} \label{thm:inf:learningParameter}
    Given $n = \tilde{\Omega}\p{\frac{d}{m^{\star} \eps^2} + \log\p{m}}$ samples from a Mallows $d$-Block distirbution $\distr$ with parameters $\pi^{\star}$ and $\vec{\phi}^{\star}$, we can estimate $\hat{\pi}$ and $\hat{\vec{\phi}}$ so that
  $\hat{\pi} = \pi^{\star}$ and $\norm{\hat{\vec{\phi}} - \vec{\phi}^{\star}}_2 \le \eps$.
\end{inftheorem}

A key observation in the proof of Theorem~\ref{thm:MallowsBlockParameterLearning} is that the sufficient statistics for a generalized Mallows model with known central ranking are provided by an $m$-variate distribution where the $i$-th coordinate is an independent \textit{truncated geometric distribution}.  Truncated geometric distributions interpolate between Bernoulli and geometric distributions. The sufficient statistics of the Mallows Block model correspond to sums of truncated geometric distributions, which interpolate between Binomial and Negative Binomial distributions. We hence believe that the study of sums of truncated geometric distribution may be of independent interest.
We should also highlight that in our approach, only the lower bound on the variance depends on Kendall's tau distance. Once we have such a bound for other exponential families, we can immediately apply our technique, e.g., to Mallows models with Spearman's Footrule and Spearman's Rank Correlation, as in \citep{Mukherjee16}.

\smallskip\paragr{Learning from one sample.} Arguably, the most interesting corollary of our tight analysis is that a single sample from a Mallows $d$-Block model with known central ranking is enough to estimate $\vec{\phi}$ within error $O\left(\sqrt{d/m^{\star}}\right)$, where again $m^{\star}$ is the minimum size of any block in the Mallows Block model. This result provides the exact rate of an asymptotic result by \cite{Mukherjee16}. The formal version of the following informal theorem can be found in Corollary~\ref{cor:singleSampleBlockModel}.

\begin{inftheorem} \label{thm:inf:learningSingleSample}
    Given a single sample from a Mallows $d$-Block distribution $\distr$ with
  known central ranking $\pi^{\star}$ and spread parameters $\vec{\phi}^{\star}$, we can estimate $\hat{\vec{\phi}}$ so that $\norm{\hat{\vec{\phi}} - \vec{\phi}^{\star}}_2 \le \tilde{O}\p{\sqrt{\frac{d}{m^{\star}}}}$.
\end{inftheorem}

\paragr{Lower Bounds.} On the lower bound side, we use Fano's inequality and show that $\Omega(\log\p{m})$ samples are necessary even for learning a simple Mallows distribution in total variation distance (Lemma~\ref{lem:permutationLowerBound}). Then, we show that $\Omega\p{\frac{d}{\eps^2}}$ samples are necessary for learning a Mallows $d$-Block distribution in total variation distance. For a formal statement of the following informal theorem we refer to Lemma~\ref{lem:learningKLDivergence:blockModel:LowerBound}.

\begin{inftheorem} \label{thm:inf:lowerBound}
    Any distribution estimation $\hat{\distr}$ that is based only on
  $o\p{\frac{d}{\eps^2} + \log\p{m}}$ samples from a Mallows $d$-Block distribution
  $\distr$
 satisfies $\TV\p{\distr, \hat{\distr}} \ge \eps$.
\end{inftheorem}

\noindent Interestingly, our lower bound uses a general way to compute the total variation distance of two distributions that belong to the same exponential family (Theorem~\ref{thm:totalVariationLowerBoundExponentialFamily}). This theorem states that the total variation of two distributions in the same exponential family is equal to the distance between their parameters times the absolute deviation of a corresponding distribution in the family. This should be compared with Theorem~\ref{thm:exponentialFamiliesProperties}.4, on the KL-divergence between two distributions in the same exponential family.
Using Theorem~\ref{thm:totalVariationLowerBoundExponentialFamily}, our lower bound boils down to showing that for some range of parameters, the absolute deviation is within a constant from the standard deviation. With this proven, we get that the total variation distance is within a constant factor from the square root of the KL-divergence, and Fano's inequality can be applied.
\smallskip

\paragr{Open Problems.} An open problem that naturally arises from the definition of the Mallows Block model is the possibility of estimating the spread parameters, even from a single sample, of the Mallows Block model when the block structure is unknown. Such results are known for the fundamental Stochastic Block model in theoretical statistics \citep{KloppTV17}. Recently, \cite{BerthetRS16} introduced the Ising Block model and proved some similar results. Another interesting question is about the minimum number of samples required to recover the block structure of the Mallows Block Model. Again, similar results are known for the Stochastic Block Model \citep{MosselNS18}.

Another research direction is to obtain lower bounds on the variance of the distance to the central ranking for other notions of distance, such as Spearman's Footrule and Spearman's Rank Correlation. Then, we can apply our general approach and obtain tight bounds on the sample complexity of learning such models and on the quality of parameter estimation from a single sample, as in \citep{Mukherjee16}.

\subsection{Related work} \label{sec:relatedWork}

There has been a significant volume of research work on algorithmic and learning problems related to our work.
In the \emph{consensus ranking problem}, a finite set $\{ \pi_1, \dots \pi_n \}$ of rankings is given, and we want to compute the ranking $\argmin_{\pi \in S_m} \sum_{i=1}^n d (\pi , \pi_i )$. This problem is known to be NP-hard~\citep{Bartholdi1989}, but it admits a polynomial-time $11/7$-approximation algorithm problem~\citep{AiChNe05} and a PTAS~\citep{KeSc07}. When the rankings are i.i.d. samples from a Mallows distribution, consensus ranking is equivalent to computing the maximum likelihood ranking, which does not depend on the spread parameter.
Intuitively, the problem of finding the central ranking should not be hard, if the probability mass is concentrated around the central ranking. \cite{MPPB07} came up with a branch and bound technique which relies on this observation. \cite{BrMo09} proposed a dynamic programming approach that computes the consensus ranking  efficiently, under the Mallows model. \cite{CaragiannisPS16} showed that the central ranking can be recovered from a logarithmic number of i.i.d. samples from a Mallows distribution (see also Theorem~\ref{thm:centralRankingLearningBlockModel}).

\cite{Mukherjee16} considered learning the spread parameter of a Mallows model based on a single sample, assuming that the central ranking is known. He studied the asymptotic behavior of his estimator and proved consistency. We strengthen this result by showing that our parameter estimator, based on single sample, can achieve optimal error for Mallows Block model (Corollary~\ref{cor:singleSampleBlockModel}).

There has been significant work either on learning a Mallows model based on partial information, e.g. partial rankings or pairwise comparisons~\citep{AdFl98,LuBo11,Busa-FeketeHS14}, or on learning generalizations of the Mallows model, such as learning mixture of Mallows models \citep{LiuM18}. Among these works, \citep{AwasthiBS14, LiuM18} seem the most relevant to our paper, since they considered learning mixtures of single parameter Mallows models in a learning setup that is similar in spirit to ours: find a model that is close to the underlying one either in the parameter space or in total variation distance based on as few sample as possible. However, the sample complexity of learning mixtures is necessarily much higher and a high degree polynomial of $1/\eps$ and $m$. Hence their results do not compare with our optimal sample complexity analysis even for the simple Mallows model case.

The parameter estimation of the Generalized Mallows Model has been examined from a practical point of view by \cite{MeiluaPPB07} but no theoretical guarantees for the sample complexity have been provided. Several ranking models are routinely used in analyzing ranking data~\citep{Mar95,Shi16}, such as Plackett-Luce model~\citep{Pla75,Luc59}, Babington-Smith model~\citep{HV93} and spectral analysis based methods~\citep{NIPS2012_4720,SibonyCJ15} and non-parametric methods~\citep{LebanonM07}. However, to our best knowledge, none of these ranking methods have been analyzed from point of distribution learning which comes with guarantee on some information theoretic distance. \citet{HajekOX14} considered the problem of learning parameters of  Plackett-Luce model and they came up with high probability bounds for their estimator that is tight in a sense that there is no algorithm which can achieve lower estimation error with fewer examples.

  \section{Preliminaries and Notation} \label{sec:model}

Small bold letters $\vec{x}$ refer to real vectors in finite dimension
$\reals^d$ and capital bold letters $\matr{A}$ refer to matrices in
$\reals^{d \times \ell}$. We denote by $x_i$ the $i$th coordinate of $\vec{x}$,
and by $A_{ij}$ the $(i, j)$th coordinate of $\matr{A}$.
For any $\vec{x}, \vec{y} \in \R^d$ we define
$L(\vec{x}, \vec{y}) = \{ \vecz \in \R^d \mid \vecz = t \vecx + (1 - t) \vecy, ~~ t \in [0, 1] \}$.
\smallskip

\paragr{Metrics between distributions.} Let $p$, $q$ be two probability measures
in the discrete probability space $(\Omega, \mathcal{A})$ then the total variation
distance between $p$ and $q$ is defined as
$\TV\p{p, q} = \frac{1}{2} \sum_{x \in \Omega} \abs{p(x) - q(x)} = \max_{A \in \mathcal{A}} \abs{p(A) - q(A)}$, and the KL-divergence between $p$ and $q$ is
defined as $\KL\p{p || q} = \sum_{x \in \Omega} p(x) \ln \p{\frac{p(x)}{q(x)}}$.
\smallskip

\paragr{Exponential Families.}
  In this section we summarize the basic definitions and properties of the
exponential families of distributions. We follow the formulation and the
expressions of \citep{Keener11, NielsenG09} where we also refer for complete
proofs of the statements presented in this section. Let $\mu$ be a measure on
$\R^d$ and also $h : \R^d \to \R_+$, $\matr{T} : \R^d \to \R^k$ be measurable
functions. We define the \textit{logarithmic partition function}
$\alpha_{\matr{T}, h} : \R^k \to \R_+$ as
$\alpha(\vec{\eta}) =  \alpha_{\matr{T}, h}(\vec{\eta}) = \ln \left( \int \exp\left( \vec{\eta}^T \matr{T}(\vec{x}) \right) h(\vec{x}) ~ d\mu(\vec{x}) \right) $.
We also define the \textit{range of natural parameters} $\mathcal{H}_{\matT, h}$
as
$\mathcal{H}_{\matT, h} = \left\{ \vec{\eta} \in \R^k \mid \alpha_{\matT, h}(\vec{\eta}) < \infty \right\}$.
The \textit{exponential family} $\expFamily(\matT, h)$ with
\textit{sufficient statistics} $\matT$, \textit{carrier measure} $h$ and natural
parameters $\vec{\eta}$ is the family of distributions
$\expFamily(\matT, h) = \left\{ \distr_{\vec{\eta}} \mid \vec{\eta} \in \mathcal{H}_{\matT, h} \right\}$
where the probability distribution $\distr_{\vec{\eta}}$ has density
\begin{equation} \label{eq:exponentialFamily}
  p_{\vec{\eta}} (\vec{x}) = \exp\left( \vec{\eta}^T \matT(\vec{x}) - \alpha(\vec{\eta}) \right) h(\vec{x}).
\end{equation}

\paragr{Truncated Geometric Distribution.} We say that a random variable $Z$
follows the \textit{truncated geometric distribution} $\tgeometric(\phi, k)$ with
parameters $k \in \N \cup \{\infty\}$ and $\phi \in [0, 1]$ if it has the
following probability mass function $p(i) = \phi^i/\sum_{j = 0}^k \phi^j$ for $i \in [0, k]$ and $0$ otherwise.

  For $k = 2$ the distribution $\tgeometric(\phi, k)$ is a Bernoulli
distribution with success probability $\phi/(1 + \phi)$. For $k = \infty$ and
$\phi \in [0, 1)$ the distribution $\tgeometric(\phi, k)$ is a geometric
distribution $\geometric(\phi)$. Observe that if we fix $k$ then
$\mathcal{E}_k = \left\{ \tgeometric(\phi, k) \mid \phi \in [0, 1] \right\}$ is
an exponential family with natural parameter $\theta = \ln (\phi)$. Again the
domain of $\phi$ changes to $\phi \in [0, 1)$ for $k = \infty$.
\smallskip

\paragr{Basic Properties of Exponential Families.} We summarize in the next
theorem the fundamental properties of exponential families. For a proof of this
theorem we refer to the Appendix \ref{sec:app:preliminaries}.
\begin{theorem} \label{thm:exponentialFamiliesProperties}
    Let $\expFamily(\matT, h)$ be an exponential family parametrized by
  $\vec{\eta} \in \R^k$ and for simplicity let
  $\alpha(\cdot) = \alpha_{\matT, h}(\cdot)$ and
  $\mathcal{H} = \mathcal{H}_{\matT, h}$ then the following hold.
  \begin{Enumerate}
    \item For all $\vec{\eta} \in \mathcal{H}$, it holds that
          \begin{equation} \label{eq:thm:exponentialProperties:1}
            \Exp_{\vec{x} \sim \distr_{\vec{\eta}}} \left[ \matT(\vec{x}) \right] = \nabla \alpha(\vec{\eta}).
          \end{equation}
    \item For all $\vec{\eta} \in \mathcal{H}$, it holds that
          \begin{equation} \label{eq:thm:exponentialProperties:2}
            \Var_{\vec{x} \sim \distr_{\vec{\eta}}} \left[ \matT(\vec{x}) \right] = \nabla^2 \alpha(\vec{\eta}).
          \end{equation}
    \item For all $\vec{\eta} \in \mathcal{H}$, $\vec{s} \in \R^d$, it holds that
          \begin{equation} \label{eq:thm:exponentialProperties:3}
            \Exp_{\vecx \sim \distr_{\vec{\eta}}} \left[ \exp\left( \vecs^T \matT(\vecx) \right) \right] = \exp\left( \alpha(\vec{\eta} + \vecs) - \alpha(\vec{\eta}) \right).
          \end{equation}
    \item For all $\vec{\eta}, \vec{\eta}' \in \mathcal{H}$, and for some $\vec{\xi} \in L(\vec{\eta}, \vec{\eta}')$ it holds that
          \begin{equation} \label{eq:thm:exponentialProperties:4}
            \KL \left( \distr_{\vec{\eta}'} || \distr_{\vec{\eta}} \right) =
            - (\vec{\eta}' - \vec{\eta})^T \nabla \alpha(\vec{\eta}) + \alpha(\vec{\eta}') - \alpha(\vec{\eta}) =
            \left( \vec{\eta}' - \vec{\eta} \right)^T \nabla^2 \alpha(\vec{\xi}) \left( \vec{\eta}' - \vec{\eta} \right).
          \end{equation}
  \end{Enumerate}
\end{theorem}

\subsection{Ranking Distributions} In this section we review the basic definitions of
exponential families over permutations. We define the single parameter Mallows
model and its generalization.
\smallskip

\paragr{Single Parameter Mallows Model.}
The Mallows model or, more specifically, Mallows $\phi$-distribution is a
parametrized, distance-based probability distribution that belongs to the family
of exponential distributions
$\Mallows_1 = \{ \distr_{\phi, \pi_0} \mid \phi \in [0, 1], \pi_0 \in S_m \}$ with
probability mass function
$p_{\phi, \pi_0} ( \pi ) =  \phi^{d ( \pi, \pi_0 )} / Z(\phi, \pi_0)$
where $\phi$ and $\pi_0$ are the parameters of the model:
$\pi_0 \in S_m$ is the location parameter also called center ranking and
$\phi \in [0,1]$ the spread parameter. Moreover, $d(\cdot,\cdot)$ is a distance
metric on permutations, which for our paper will be the Kendall tau distance ,
that is, the number of discordant item pairs $d_K(\pi,\pi')= \sum_{1 \leq i < j \leq m} \chara\left\{(\pi(i) - \pi(j))(\pi'(i) - \pi'(j)) < 0 \right\}$.

  The normalization factor in the definition of the model is equal to
$Z(\phi, \pi_0) = \sum_{ \pi \in S_n} \phi^{d ( \pi, \pi_0 )}$.
When the distance metric $d$ is the Kendall tau distance we have $Z(\phi, \pi_0) = Z(\phi) = \prod_{i = 1}^{m - 1} \sum_{j = 0}^i \phi^{j}$.
Observe that the family of distributions as stated is not an exponential
family because of the location parameter $\pi_0$. If we fix the
permutation parameter then the family
$\Mallows_1(\pi_0) = \{ \distr_{\phi, \pi_0} \mid \phi \in [0, 1] \}$ is an
exponential family with natural parameter $\theta = \ln \phi$.
\medskip

\paragr{Generalized Mallows Model.} One of the most famous generalizations of
Mallows model is the one introduced by \cite{FlignerV1986} with the name
\textit{Generalized Mallows Model}. We define $V_j(\sigma, \pi)$ to be the
\textit{number of discordant item pairs involving item $j$}, i.e.
$V_j(\sigma, \pi) = \sum_{1 \le i < j} \mathbf{1}\{ (\sigma_i - \sigma_j) (\pi_i - \pi_j) < 0 \}$.
The generalized Mallows family of distribution
$\Mallows_m = \{ \distr_{\vec{\phi}, \pi_0} \mid \vec{\phi} \in [0, 1]^m, \pi_0 \in S_m\}$
with parameters $\pi_0 \in S_m$ and
$\vec{\phi} = (\phi_1, \dots, \phi_m) \in [0, 1]^m$ is defined as the
probability measure over $S_m$ with probability mass function
$p_{\vec{\phi}, \pi_0} (\pi) = \prod_{i = 1}^m \phi_i^{V_i(\pi, \pi_0)} / Z(\vec{\phi}, \pi_0)$.
One important property of the generalized mallows model
when the distance metric $d$ is the Kendall tau distance is that the random
variables $Y_i = V_i(X, \pi)$ where $X \sim \distr_{\vec{\phi}, \pi_0}$ are
independent. This follows from the following decomposition lemma of the partition
function $Z(\vec{\phi})$. For the proof of Lemma
\ref{lem:partitionFunctionDecomposition} we refer to the Appendix
\ref{sec:app:preliminaries}.

\begin{lemma} \label{lem:partitionFunctionDecomposition}
  When $d = \mathrm{Kendall~tau~distance}$, we have that
  $Z(\vec{\phi}, \pi_0) = Z(\vec{\phi}) = \prod_{i = 1}^m Z_i(\phi_i)$,
  where $Z_i(x) = \sum_{j = 0}^{i - 1} x^j$.
\end{lemma}

  In Section \ref{sec:blockModel} we introduce the Mallows Block Model that
interpolates between the single parameter and the generalized Mallows model.

\subsection{Fano's Inequality} \label{sec:FanoInequality}
  In this section we present Fano's inequality which is our main technical tool
for proving lower bounds on the sample complexity of learning Mallows Block Models.
For this, let $\support$ denote some finite set.
\medskip

\paragr{Maximum Risk of an Estimator.} Let $\family$ be a family of distributions and
assume that we have access to $n$ i.i.d. samples
$\vec{x} = (x_1, \dots, x_n) \sim f^n \in \family$. Let $\hat{f} : \support^n \to \Delta_{\support}$. Then the
\textit{maximum risk of $\hat{f}$ with respect to the family $\family$} is equal to
\begin{equation} \label{eq:estimatorRiskDefinition}
  \risk_n(\hat{f}, \family) = \sup_{f \in \family} \Exp_{\vec{x} \sim f^n} \left[ \TV(\hat{f}(\vecx), f) \right].
\end{equation}

\paragr{Minimax Risk.} Let $\family$ be a family of distributions and
assume that we have access to $n$ i.i.d. samples
$\vec{x} = (x_1, \dots, x_n) \sim f^n \in \family$. Let also
$\Omega = \{ \hat{f} : \support^n \to \Delta_{\support} \}$. Then we define the
\textit{minimax risk} of the family $\family$ as
\begin{equation} \label{eq:minimaxRiskDefinition}
  \risk_n(\family) = \inf_{\hat{f} \in \Omega} \risk_n(\hat{f}, \family).
\end{equation}

\noindent We can now state Fano's Inequality as presented by \cite{Yu1997}.

\begin{theorem}[Lemma 3 in \citep{Yu1997}] \label{thm:FanoInequality}
    Let $\family$ be a finite family of densities such that
  \[ \inf_{f, g \in \family : f \neq g} \TV(f, g) \ge \alpha, ~~~~~~~
     \sup_{f, g \in \family : f \neq g} \KL(f || g) \le \beta, \]
  \noindent then it holds that
  \[ \risk_n(\family) \ge \frac{\alpha}{2} \left(1 - \frac{n \beta + \ln 2 }{\ln \abs{\family}}\right).  \]
\end{theorem}
  \clearpage
\section{Concentration Inequality  and Total Variation of Exponential Families} \label{sec:concentration}

We shall prove a concentration inequality for the sufficient
statistics of an exponential family. This concentration inequality will be the
basic building block for the general learning algorithm for exponential
inequalities that we will present in the next section. Then we prove an exact
formula for the total variation distance between two distributions that belong to
the same exponential family.

\begin{theorem} \label{thm:exponentialFamiliesConcentration}
   Let $\expFamily(T, h)$ be an exponential family with natural parameter
  $\eta \in \R$, logarithmic partition function $\alpha$ and range of parameters
  $\mathcal{H}$. Then the following concentration inequality
  holds for all $\eta, \eta' \in \mathcal{H}$
  \begin{equation} \label{eq:thm:exponentialConcentration1D}
    \Prob_{\vecx \sim \distr^n_{\eta}} \left( \left( \frac{1}{n} \sum_{i = 1}^n T(x_i) \right) (\eta' - \eta) \ge \Exp_{y \sim \distr_{\eta'}} \left[ T(y) \right] (\eta' - \eta) \right) \le \exp\left( - \KL \left( \distr_{\eta'} || \distr_{\eta} \right) n \right).
  \end{equation}
\end{theorem}

\begin{proof}
    We give the proof for $\eta' > \eta$ and the case $\eta' < \eta$ can be handled
  respectively. Let $s > 0$, $\eta' > \eta$ and for simplicity
  $p = \Prob_{\vecx \sim \distr^n_{\eta}} \left( \left( \frac{1}{n} \sum_{i = 1}^n T(x_i) \right) \ge \Exp_{y \sim \distr_{\eta'}} \left[ T(y) \right] \right)$
  then it holds that
  \begin{align*}
    p
        & = \Prob_{\vecx \sim \distr^n_{\eta}} \left( \exp\left( s \left( \sum_{i = 1}^n T(x_i) \right) \right) \ge \exp \left( s \cdot n \Exp_{y \sim \distr_{\eta'}} \left[ T(y) \right] \right) \right) & \\
    & \le \frac{\Exp_{x_i \sim \distr_{\eta}}\left[ \exp \left( s \sum_{i = 1}^n T(x_i) \right) \right]}{\exp \left( s \cdot n \Exp_{y \sim \distr_{\eta'}} \left[ T(y) \right] \right)} & \text{(Markov's Inequality)} \\
    & = \left( \frac{\Exp_{x \sim \distr_{\eta}}\left[ \exp \left( s  T(x) \right) \right]}{\exp \left( s \Exp_{y \sim \distr_{\eta'}} \left[ T(y) \right] \right)} \right)^n & \text{(Independence of $x_i$'s)} \\
    & = \left( \frac{\exp \left( \alpha\left( \eta + s \right) - \alpha(\eta) \right)}{\exp \left( s \dot{\alpha}(\eta') \right)} \right)^n = \exp\left( - \left(s \dot{\alpha}(\eta') - \alpha(\eta + s) + \alpha(\eta) \right) n \right) & \text{(By \eqref{eq:thm:exponentialProperties:1}, \eqref{eq:thm:exponentialProperties:3})}
          \end{align*}
  \noindent Now we define the function
  $f(s) = s \dot{\alpha}(\eta') - \alpha(\eta + s) + \alpha(\eta)$.
  The second derivative of $f$ is
  $f''(s) = - \ddot{\alpha}(\eta + s)$. From
  \eqref{eq:thm:exponentialProperties:2} we conclude that
  $\ddot{\alpha}(\eta + s) \ge 0$ and hence $f''(s) \le 0$
  which implies that $f$ is a concave function. Hence $f$ achieves its maximum
  for at $s^*$ such that $f'(s) = 0$. But
  $f'(s) = \dot{\alpha}(\eta') - \dot{\alpha}(\eta + s)$
  which implies that for $s^* = \eta' - \eta$ it holds that $f'(s^*) = 0$.
  Therefore the optimal bound of the above form is achieved for
  $s = \eta' - \eta$. Hence we have the following
  \begin{align*}
    p & \le \exp\left( - \left(s^* \dot{\alpha}(\eta') - \alpha(\eta + s^*) + \alpha(\eta) \right) n \right) \overset{\eqref{eq:thm:exponentialProperties:4}}{=} \exp\left( - \KL \left(\distr_{\eta'} || \distr_{\eta}\right) n \right)
  \end{align*}
  \noindent which concludes the proof.
\end{proof}

\noindent The following useful corollary of Theorem
\ref{eq:thm:exponentialConcentration1D} can be obtained if we apply Pinsker's
inequality to the right hand side of \eqref{eq:thm:exponentialConcentration1D}.

\begin{corollary} \label{cor:exponentialFamiliesConcentrationTV}
    Let $\expFamily(T, h)$ be an exponential family with natural parameter
  $\eta \in \R$, logarithmic partition function $\alpha$ and range of parameters
  $\mathcal{H}$. Then the following concentration inequality
  holds for all $\eta, \eta' \in \mathcal{H}$
  \begin{equation} \label{eq:cor:exponentialConcentration1DTV}
    \Prob_{\vecx \sim \distr^n_{\eta}} \left( \left( \frac{1}{n} \sum_{i = 1}^n T(x_i) \right) (\eta' - \eta) \ge \Exp_{y \sim \distr_{\eta'}} \left[ T(y) \right] (\eta' - \eta) \right) \le \exp\left( - 2 \mathrm{d}^2_{\mathrm{TV}} \left( \distr_{\eta'}, \distr_{\eta} \right) n \right).
  \end{equation}
\end{corollary}

  We now move to proving an exact formula for
$\TV\p{\distr_{\vec{\eta}}, \distr_{\vec{\eta}'}}$. For the proof of
Theorem \ref{thm:totalVariationLowerBoundExponentialFamily} we refer to the
Appendix \ref{sec:app:concentration}.

\begin{theorem} \label{thm:totalVariationLowerBoundExponentialFamily}
    Let $\expFamily(\matT, h)$ be an exponential family with natural parameters
  $\vec{\eta}$. If
  $\distr_{\vec{\eta}}$, $\distr_{\vec{\eta}'}$ $\in \expFamily(\matT, h)$, with
  then for some $\vec{\xi} \in L(\vec{\eta}, \vec{\eta}')$ it holds that
  \[ \TV\p{\distr_{\vec{\eta}}, \distr_{\vec{\eta}'}} = \Exp_{\vecx \sim \distr_{\vec{\xi}}} \b{\sign \p{\distr_{\vec{\eta}}(\vecx) - \distr_{\vec{\eta}'}(\vecx)} \p{\vec{\eta} - \vec{\eta}'}^T \p{\matT(\vecx) - \Exp_{\vecy \sim \distr_{\vec{\xi}}}\b{\matT(\vecy)}}}. \]
\end{theorem}

  To give some intuition about Theorem \ref{thm:totalVariationLowerBoundExponentialFamily},
consider the single dimensional case with $\eta' \to \eta$ and $\eta \ge \eta'$. In this case, it
is easy to see that the sign of
$\p{\matT(\vecx) - \Exp_{\vecy \sim \distr_{\vec{\xi}}}\b{\matT(\vecy)}}$ and
$\p{\distr_{\vec{\eta}}(\vecx) - \distr_{\vec{\eta}'}(\vecx)}$ are the same and hence the
expression becomes
$\p{\eta - \eta'} \Exp_{\vecx \sim \distr_{\vec{\xi}}} \b{\abs{\matT(\vecx) - \Exp_{\vecy \sim \distr_{\vec{\xi}}}\b{\matT(\vecy)}}}$.
This gives the intuition that the total variation of two distribution in the same exponential
family, with parameters sufficiently close, is equal to the distance between their parameters times
the \textit{absolute deviation} of a corresponding distribution in the family. This should be compared with
Theorem \ref{thm:exponentialFamiliesProperties}.4, on the KL-divergence between
two distributions in the same exponential family. The single dimensional version
Theorem \ref{thm:exponentialFamiliesProperties}.4 states that the KL-divergence is
equal to the square difference of their parameters multiplied by the variance of a
corresponding distribution inside the exponential family. Since the standard
deviation is greater than the absolute deviation this conclusion resembles the
well known Pinsker's inequality. Furthermore, in a lot of exponential families, e.g. Gaussian
distributions, the absolute deviation is only a constant fraction away from the standard deviation
which indicates the existence of a \textit{converse Pinsker's inequality} in these settings.

  \section{Warm-up: Learning Single Parameter Mallows Model}
\label{sec:singleParameter}

  In this section we give a simple algorithm and prove its sample complexity for
learning the parameters $(\phi, \pi_0)$ of a single parameter distribution
$\distr_{\phi, \pi_0} \in \Mallows_1$ given i.i.d. samples $\pi_1, \dots, \pi_n$
from $\distr$. We also provide bounds for learning the distribution
$\distr_{\phi, \pi_0}$ in total variation distance. As we will see if the
central ranking $\pi_0$ is known then an accurate estimation of $\phi$ is
possible hence giving an alternative proof of a phenomenon proved by
\cite{Mukherjee16}.

\subsection{Parameter Estimation} \label{sec:parameterEstimationSingleParameter}

  For the single parameter Mallows model the sample complexity of estimating the
central ranking has been identified in \cite{CaragiannisPS16} as we see in the
next theorem. We focus on the case where the ranking distance is the Kendall
tau distance $d_K$.

\begin{theorem}[\citep{CaragiannisPS16}] \label{thm:centralRankingLearning1D}
    For any $\pi_0 \in S_m$ and any $\phi \in [0, 1 - \gamma]$, there exists a
  polynomial time estimator $\hat{\pi}$ such that given
  $n = \Theta(\frac{1}{\gamma} \log(m / \delta))$ i.i.d. samples
  $\pi_1, \dots, \pi_n \sim \distr_{\phi, \pi_0}$ satisfies
  $\Prob_{\vec{\pi} \sim \distr_{\phi, \pi_0}^n} \left( \hat{\pi} \neq \pi_0 \right) \le \delta$.
  Moreover, if $n = o(\log(m / \delta))$ then for any estimator $\hat{\pi}$
  there exists a distribution $\distr_{\phi, \pi_0}$ such that
  $\Prob_{\vec{\pi} \sim \distr_{\phi, \pi_0}^n} \left( \hat{\pi} \neq \pi_0 \right) > \delta$.
\end{theorem}

  Hence it remains to estimate the parameter $\phi$ if we have the knowledge of
the central ranking $\pi_0$. As we explained in the definition of Mallows model
when the central ranking is known the family of distributions
$\Mallows_1(\pi_0)$ is a single parameter exponential family. The sufficient
statistic of this family is $T(\pi) = d_K(\pi, \pi_0)$. The natural parameter of
$\Mallows_1(\pi_0)$ is the parameter $\theta = \ln \phi$ and logarithmic
partition function $\alpha(\theta) = \ln \left( Z(e^{\theta}) \right)$.

\begin{theorem} \label{thm:singleMallowsLearning}
    For any $\pi_0 \in S_m$, $\phi^{\star} \in [0, 1- \gamma]$, $\eps, \delta > 0$
  there exist estimators $\hat{\pi}, \hat{\phi}$ that can be computed in
  polynomial time from i.i.d. samples
  $\vec{\pi} \sim \distr^n_{\phi^{\star}, \pi_0}$ such that if
  $n \ge \Omega\p{\frac{\log\p{1/\delta}}{m \eps^2} + \frac{\log\p{m/\delta}}{\gamma}}$, then
  \[ \Prob_{\vec{\pi} \sim \distr_{\phi^{\star}, \pi_0}^n} \left( \left(\hat{\pi} = \pi_0\right) \wedge \left(\hat{\phi} \in \left[\phi^{\star} - \eps, \phi^{\star} + \eps\right]\right) \right) \ge 1 - \delta. \]
  \noindent In the case where $\pi_0$ is known for
  $\phi^{\star} \in [0, 1]$ then there exists an estimator
  $\hat{\phi}$ that can be computed in polynomial time such that if $n \ge \Omega\p{\frac{\log\p{1/\delta}}{m \eps^2}} $, then
  \[ \Prob_{\vec{\pi} \sim \distr_{\phi^{\star}, \pi_0}^n} \left( \hat{\phi} \in \left[\phi^{\star} - \eps, \phi^{\star} + \eps\right] \right) \ge 1 - \delta. \]
\end{theorem}

\noindent Theorem \ref{thm:singleMallowsLearning} follows from the more general
Theorem \ref{thm:MallowsBlockParameterLearning} and hence we postpone its proof
for the Section \ref{sec:blockModel}. One interesting thing to point out though
from Theorem \ref{thm:singleMallowsLearning} is that in the case where $\pi_0$ is
known, Theorem \ref{thm:singleMallowsLearning} provides accuracy for the parameter
$\phi$ that goes to $0$, even with $n = 1$ sample, as the size of the permutation
goes to infinity, i.e. $m \to \infty$. This was observed before by
\cite{Mukherjee16} but no explicit rates as the ones we provide, were provided. We
summarize our result for $n = 1$ sample in the following corollary, which
immediately follows from Theorem \ref{thm:singleMallowsLearning}.

\begin{corollary} \label{cor:singleSampleSingleParameter}
    For any known $\pi_0 \in S_m$, any $\phi^{\star} \in [0, 1]$ and
  $\delta > 0$, there exists an estimator $\hat{\phi}$ that can be computed in
  polynomial time from one sample $\pi \sim \distr_{\phi^{\star}, \pi_0}$ such
  that
  \[ \Prob_{\pi \sim \distr_{\phi^{\star}, \pi_0}} \left( \hat{\phi} \in \left[\phi^{\star} - \eps, \phi^{\star} + \eps\right] \right) \ge 1 - \delta\,, \ \ \     \mbox{where $\eps = O\p{\sqrt{\frac{\log\p{1/\delta}}{m}}}$}. \]
\end{corollary}

\subsection{Learning in KL and TV Distance} \label{sec:singleMallows:learningKL}

  The upper bound on the number of samples that we need to learn the
distribution $\distr_{\phi^{\star}, \pi_0}$ in KL and TV distance follows
from Theorem \ref{thm:exponentialFamiliesConcentration} and Theorem
\ref{thm:centralRankingLearning1D} as we show in the more general Theorem
\ref{thm:learningKLDivergence:blockModel}. To finish this section we focus on
proving the lower bound for learning in TV distance. The lower bound for learning
the single parameter $\phi$ follows again from the corresponding lower bound of
Section \ref{sec:blockModel} and hence the term $\frac{1}{\eps^2}$ in the sampling
complexity necessary. In the next lemma we prove that the term $\log\p{m}$ is
also necessary.

\begin{lemma} \label{lem:permutationLowerBound}
   For any $n = o(\log(m))$ it holds that
  \[ \risk_n(\Mallows_1) \ge 1/16.  \]
\end{lemma}

\noindent For the proof of Lemma \ref{lem:permutationLowerBound} we refer to the
Appendix \ref{sec:app:singleParameter}.
    \clearpage
\section{Learning Mallows Block Model} \label{sec:blockModel}

  We start this section with properties of the Generalized Mallows Model as it is
defined in Section \ref{sec:model}. Then we move to the definition of the Mallows
Block Model and the presentation of our main results. We remind the reader that
the generalized Mallows family of distribution is
$\Mallows_m = \{ \distr_{\vec{\phi}, \pi_0} \mid \vec{\phi} \in [0, 1]^m, \pi_0 \in S_m\}$
with parameters $\pi_0 \in S_m$ and
$\vec{\phi} = (\phi_1, \dots, \phi_m) \in [0, 1]^m$ is defined as the
probability measure over $S_m$ with probability mass function that using Lemma
\ref{lem:partitionFunctionDecomposition} is equal to
\begin{equation} \label{eq:generalizedMallowsModelProduct}
  p_{\vec{\phi}, \pi_0}(\pi) = \prod_{i = 1}^m \frac{\phi_i^{V_i(\pi, \pi_0)}}{Z_i(\phi_i)}.
\end{equation}

\noindent We define now the random variables
$Y_i = V_i(\pi, \pi_0)$
where $\pi \sim \distr_{\vec{\phi}, \pi_0}$ which are the sufficient statistics
for $\distr_{\vec{\phi}, \pi_0}$ when $\pi_0$ is known. It is easy to observe from
\eqref{eq:generalizedMallowsModelProduct}
that the probability mass function of the vector $(Y_1, \dots, Y_m)$ is
\begin{equation} \label{eq:vectorOfTruncatedGeometric}
  \Prob(Y_1 = y_1, \dots, Y_m = y_m) = \left( \frac{\phi_1^{y_1}}{Z_1(\phi_1)} \right) \cdots \left( \frac{\phi_m^{y_m}}{Z_m(\phi_m)} \right)
   = \Prob(Y_1 = y_1) \cdots \Prob(Y_m = y_m)
\end{equation}
\noindent and hence the random variables $Y_i$ are independent. Observe also from
the probability mass function and the definition of the \textit{truncated
geometric distribution} in Section \ref{sec:model} that
$Y_i \sim \tgeometric(\phi_i, i - 1)$. To formally summarize this observation we
define $\distr_{\vec{\phi}}$ to be the multivariate distribution
$(Z_1, \dots, Z_m)$, where $Z_i \sim \tgeometric(\phi_i, i - 1)$. The following
lemma relates the distribution $\distr_{\vec{\phi}}$ with the distribution
$\distr_{\vec{\phi}, \pi_0}$ when the central ranking $\pi_0$ is known. For the
proof we refer to the Appendix \ref{sec:app:blockModel}.

\begin{lemma} \label{lem:truncatedGeometricAndMallows}
    Let $\pi_0 \in S_m$ and $\vec{\phi}, \in [0,1]^m$. Let also $R_{\vec{\phi}}$
  be the support of the distribution $\distr_{\vec{\phi}}$ and
  $R_{\vec{\phi}, \pi_0}$ the support of the distribution
  $\distr_{\vec{\phi}, \pi_0}$. Then there exists a bijective map
  $h : R_{\vec{\phi}, \pi_0} \to R_{\vec{\phi}}$ such that for any
  $\sigma \in R_{\vec{\phi}, \pi_0}$ it holds that
  $Prob_{\pi \sim \distr_{\vec{\phi}, \pi_0}} \left( \pi = \sigma \right) = \Prob_{\vec{y} \sim \distr_{\vec{\phi}}} \left( \vec{y} = h(\sigma) \right)$.
  In particular,
  $\TV\p{\distr_{\vec{\phi}, \pi_0}, \distr_{\vec{\phi}', \pi_0}} = \TV\p{\distr_{\vec{\phi}}, \distr_{\vec{\phi}'}}$
  and
  $\KL\p{\distr_{\vec{\phi}, \pi_0} || \distr_{\vec{\phi}', \pi_0}} = \KL\p{\distr_{\vec{\phi}} || \distr_{\vec{\phi}'}}$.
\end{lemma}

\noindent The above lemma reduces the problem of learning the Generalized Mallows
distribution $\distr_{\vec{\phi}, \pi_0}$ to the learning of the central ranking
$\pi_0$ and the distribution $\distr_{\vec{\phi}}$.
\smallskip

\paragr{Mallows Block Model.} The motivation of Mallows Block Model is to
incorporate setting where some group of alternatives have the same probability of being
misplaced hence they have the same parameter $\phi_i$, but not all alternatives
have the same probability of being misplaced as in the single parameter Mallows
model. As we will explore in this section, the knowledge of the groups of
alternatives with the same parameter can significantly decrease the number of
samples needed  to learn the parameters of the model. In the extreme case, when
the size of the groups of alternatives is large enough, we can get very good rates
even from just one samples from the distribution as we already discussed in
Corollary \ref{cor:singleSampleSingleParameter}. The Mallows Block Model with $d$
parameters is the family of distributions
\[ \Mallows_d(\matr{B}) = \{ \distr_{\vec{\phi}, \pi_0, \matr{B}} \mid \vec{\phi} \in [0, 1]^d, \pi_0 \in S_m\} \]
\noindent where $\matr{B} = \{B_1, \dots, B_d\}$ is a partitioning of the set
$[m]$. Each distribution $\distr_{\vec{\phi}, \pi_0, \matB}$ is defined as a
probability measure over $S_m$ with the following probability mass function
\begin{equation} \label{eq:massFunctionBlockModel}
  p_{\vec{\phi}, \pi_0, \matB} (\pi) = \frac{1}{Z(\vec{\phi}, \pi_0, \matB)} \prod_{i = 1}^d \phi_i^{\sum_{j \in B_i} V_j(\pi, \pi_0)}.
\end{equation}

\noindent Again using Lemma \ref{lem:partitionFunctionDecomposition} we have that
$Z(\vec{\phi}, \pi_0, \matB) = Z(\vec{\phi}, \matB) = \prod_{i = 1}^d \p{\prod_{j \in B_i} Z_j(\phi_i)}$. The sufficient statistics of
$\distr_{\vec{\phi}, \pi_0, \matr{B}}$, when $\pi_0, \matB$ are known, is the $d$
dimensional vector $\matT(\pi, \pi_0, \matB)$ where
$T_i(\pi, \pi_0, \matB) = \sum_{j \in B_i} V_j(\pi, \pi_0)$.
We define the distribution $\distr_{\vec{\phi}, \matB}$ to be the distribution of
the random vector $(Z_1, \dots, Z_m)$ where $Z_j \sim \tgeometric(\phi_i, j - 1)$
are independent and $i$ satisfies $j \in B_i$.

  One important parameter of the Mallows Block Model are the sizes of the sets
$B_i$ in the partition $\matB$ of $[m]$. For this reason we define
$m_i = \abs{B_i}$ and $m^{\star} = \min_{i \in [d]} \abs{B_i}$.

\subsection{Parameter Estimation in Mallows Block Model} \label{sec:parameterEstimationBlockModel}

  We start with the estimation of the central ranking. Since the single
parameter Mallows model is a special case of the Mallows Block Model the lower
bound of \cite{CaragiannisPS16} presented in Theorem
\ref{thm:centralRankingLearning1D} still holds, and thus $\Omega(\log(m))$
samples are necessary. The upper bound we present in Theorem
\ref{thm:centralRankingLearningBlockModel}. Its proof is deferred to
Appendix \ref{sec:app:blockModel}.

\begin{theorem} \label{thm:centralRankingLearningBlockModel}
    For any $\pi_0 \in S_m$, any $\vec{\phi} \in [0, 1 - \gamma)^d$, any known partition
  $\matB$ of $[m]$, there exists a polynomial time computable estimator
  $\hat{\pi}$ such that given
  $n = \Theta(\frac{1}{\gamma}\log(m / \delta))$ i.i.d. samples
  $\vec{\pi} = (\pi_1, \dots, \pi_n) \sim \distr_{\vec{\phi}, \pi_0, \matB}$ satisfies
  $\Prob_{\vec{\pi} \sim \distr_{\vec{\phi}, \pi_0}^n} \left( \hat{\pi} \neq \pi_0 \right) \le \delta$.
  Moreover, if $n = o(\log(m / \delta))$ then for any estimator $\hat{\pi}$ there
  exists a distribution $\distr_{\vec{\phi}, \pi_0, \matB}$ such that
  $\Prob_{\vec{\pi} \sim \distr_{\vec{\phi}, \pi_0}^n} \left( \hat{\pi} \neq \pi_0 \right) > \delta$.
\end{theorem}

What remains is to estimate the vector of parameters $\vec{\phi}$ assuming
the knowledge of the central ranking $\pi_0$. As we explained in the definition
of Mallows Block Model when the central ranking is known the family of
distributions $\Mallows_d(\matB, \pi_0)$ is an exponential family. The
sufficient statistics of this family are
$T_i(\pi, \pi_0, \matB) = \sum_{j \in B_i} V_j(\pi, \pi_0)$. The natural
parameters of $\Mallows_d(\matB, \pi_0)$ is the vector of parameters
$\vec{\theta} \in \R_-^d$ where $\theta_i = \ln\p{\phi_i}$ and logarithmic
partition function
$\alpha(\vec{\theta}, \matB) = \ln \left( Z(\vec{\phi}, \matB) \right)$. We may
simplify the notation $\alpha(\vec{\theta}, \matB)$ to
$\alpha(\vec{\theta})$ when $\matB$ is clear from the context.

\begin{theorem} \label{thm:MallowsBlockParameterLearning}
    For any $\pi_0 \in S_m$, $\vec{\phi}^{\star} \in [0, 1 - \gamma)^d$, any fixed
  partition $\matB$ of $[m]$ with $\abs{\matB} = d$ and any $\eps, \delta > 0$
  there exist estimators $\hat{\pi}, \hat{\vec{\phi}}$ that can be computed in
  polynomial time from i.i.d. samples
  $\vec{\pi} \sim \distr^n_{\vec{\phi}^{\star}, \pi_0, \matB}$
  such that if
  $n \ge \Omega\p{\frac{d \log\p{d/\delta}}{m^{\star} \eps^2} + \frac{\log\p{m/\delta}}{\gamma}}$,
  where $m^{\star} = \min_{i \in [d]} \abs{B_i}$, then
  \[ \Prob_{\vec{\pi} \sim \distr_{\vec{\phi}^{\star}, \pi_0, \matB}^n} \left( \left(\hat{\pi} = \pi_0\right) \wedge \left(\norm{\hat{\vec{\phi}} - \vec{\phi}^{\star}}_2 \le \eps \right) \right) \ge 1 - \delta. \]
  \noindent In the case where $\pi_0$ is known and
  $\vec{\phi}^{\star} \in [0, 1]^d$ then there exists an estimator
  $\hat{\vec{\phi}}$ that can be computed in polynomial time such that if
  $n \ge \Omega\p{\frac{d \log\p{d/\delta}}{m^{\star} \eps^2}}$
  \noindent then
  \[ \Prob_{\vec{\pi} \sim \distr_{\vec{\phi}^{\star}, \pi_0, \matB}^n} \left( \norm{\hat{\vec{\phi}} - \vec{\phi}^{\star}}_2 \le \eps \right) \ge 1 - \delta. \]
\end{theorem}

\noindent As a corollary of Theorem \ref{thm:MallowsBlockParameterLearning} we also have that when $\pi_0$ is known even one sample is sufficient to consistently learn all the parameters $\vec{\phi}^{\star}$ as the size of the
smaller block of $\matB$ goes to infinity.

\begin{corollary} \label{cor:singleSampleBlockModel}
    Let $\pi_0 \in S_m$, $\vec{\phi}^{\star} \in [0, 1]^d$, $\delta > 0$ and
  a partition $\matB$ of $[m]$ with $\abs{\matB} = d$, there exist an
  estimator $\hat{\vec{\phi}}$ that can be computed in polynomial time from a
  sample $\pi \sim \distr_{\vec{\phi}^{\star}, \pi_0, \matB}$ such that
  \[ \Prob_{\pi \sim \distr_{\vec{\phi}^{\star}, \pi_0, \matB}} \left( \norm{\hat{\vec{\phi}} - \vec{\phi}^{\star}}_2 \ge \eps \right) \ge 1 - \delta \]
  \noindent where $\eps = O\p{\sqrt{d/m^{\star}} \cdot \sqrt{\log\p{d/\delta}}}$
  and $m^{\star} = \min_{i \in [d]} \abs{B_i}$.
\end{corollary}

\begin{prevproof}{Theorem}{thm:MallowsBlockParameterLearning}
\textbf{(Sketch)}
    From Theorem \ref{thm:centralRankingLearningBlockModel} we focus on the
  estimation of the parameters $\vec{\phi}^{\star}$. We describe the
  intuition for the single parameter Mallows Model and we defer the full
  proof to Appendix \ref{sec:app:thm:MallowsBlockParameterLearning}. Let
  $\vec{\phi}^{\star} = \phi^{\star} \in [0, 1]$. Once the central ranking
  is known the distribution is an exponential family and let $T(\pi)$ be its
  sufficient statistics. It is not hard to prove that
  $\Exp_{\vec{\pi} \sim \distr_{\phi, \pi_0}}\b{T(\pi)}$ is an increasing
  function of $\phi$. Therefore, it follows with a simple argument, that the
  better we estimate
  $\Exp_{\vec{\pi} \sim \distr_{\phi^{\star}, \pi_0}}\b{T(\pi)}$ the better
  we can estimate $\phi^{\star}$. Now the main idea of our proof is to use
  the general concentration inequality of Theorem
  \ref{thm:exponentialFamiliesConcentration} to bound the accuracy that we
  can estimate
  $\Exp_{\vec{\pi} \sim \distr_{\phi^{\star}, \pi_0}}\b{T(\pi)}$. As it is
  clear from the form of the concentration inequality
  \eqref{eq:thm:exponentialConcentration1D}, to get good enough concentration
  we have to prove a strong lower bound on the KL-divergence of two
  distributions in the family. From \eqref{eq:thm:exponentialProperties:2}
  this reduces to proving a lower bound on the variance of a distribution
  in the family with parameter $\psi$ that is very close to $\phi^{\star}$.
  Such a good lower bound is not always possible to prove and we have to
  consider some cases. But in the main case a very careful lower bound of the
  variance in combination with \eqref{eq:thm:exponentialConcentration1D}
  gives the sample complexity upper bound.
\end{prevproof}

\subsection{Learning in KL-divergence and Total Variation Distance}

  In this section we will describe how we can use the concentration inequality
that we proved in Section \ref{sec:concentration} to learn a distribution
$\distr_{\vec{\phi}^{\star}, \pi_0, \matB}$ in KL-divergence from i.i.d.
samples. We also prove a lower bound that matches the upper bound up to a
$\log(d)$ factor.

\begin{theorem} \label{thm:learningKLDivergence:blockModel}
    For any $\pi_0 \in S_m$, $\vec{\phi}^{\star} \in [0, 1]^d$, any fixed
  partition $\matB$ of $[m]$ with $\abs{\matB} = d$ and any
  $\eps, \delta > 0$ there exist estimators $\hat{\pi}, \hat{\vec{\phi}}$ that
  can be computed in polynomial time from i.i.d. samples
  $\vec{\pi} \sim \distr^n_{\vec{\phi}^{\star}, \pi_0, \matB}$
  such that if
  $n \ge \Omega\p{\frac{d}{\eps^2}\log\p{d/\delta} + \log\p{m}}$,
  \noindent then
  \[ \Prob_{\vec{\pi} \sim \distr^n_{\vec{\theta}^{\star}, \pi_0, \matB}} \left( \KL\p{\distr_{\hat{\vec{\phi}}, \hat{\pi}, \matB} || \distr_{\vec{\phi}^{\star}, \pi_0, \matB}} \le \eps^2 \right) \ge 1 - \delta \]
  and hence
  $\Prob_{\vec{\pi} \sim \distr^n_{\vec{\theta}^{\star}, \pi_0, \matB}} \left( \TV\p{\distr_{\hat{\vec{\phi}}, \hat{\pi}, \matB}, \distr_{\vec{\phi}^{\star}, \pi_0, \matB}} \le \eps \right) \ge 1 - \delta$.

  \noindent Furthermore, for any $m \in \N$ there exists an $\eps_0 > 0$ such that
  for all $0 < \eps \le \eps_0$ and all functions
  $\distr : S_m^n \to \Delta_{S_m}$ with $n = o\p{\frac{d}{\eps^2}}$ there exists
  $\pi_0 \in S_m$, partition $\matB$ of $[m]$ and $\vec{\phi}^{\star} \in [0, 1]^d$ such that
  \[ \Prob_{\vec{\pi} \sim \distr^n_{\vec{\phi}^{\star}, \pi_0, \matB}} \p{ \TV\p{\distr_{\vec{\phi}^{\star}, \pi_0, \matB}, \distr(\vec{\pi})} \ge 2 \eps} \ge 1/3. \]
\end{theorem}

\noindent The proof of Theorem \ref{thm:learningKLDivergence:blockModel} is
based on two lemmas, one for the upper bound and one for the lower bound, that
we present here and the Lemma \ref{lem:permutationLowerBound} that we presented
in Section \ref{sec:singleParameter}. For the proofs of Lemma
\ref{lem:learningKLDivergence:blockModel:UpperBound} and Lemma
\ref{lem:learningKLDivergence:blockModel:LowerBound} we refer to the Appendix
\ref{sec:app:blockModel}.

\begin{lemma} \label{lem:learningKLDivergence:blockModel:UpperBound}
    For any $\pi_0 \in S_m$, $\vec{\phi}^{\star} \in [0, 1]^d$, any fixed
  partition $\matB$ of $[m]$ with $\abs{\matB} = d$ and any
  $\eps, \delta > 0$ there exist estimators $\hat{\pi}, \hat{\vec{\phi}}$ that
  can be computed in polynomial time from i.i.d. samples
  $\vec{\pi} \sim \distr^n_{\vec{\phi}^{\star}, \pi_0, \matB}$
  such that if
  $n \ge \Omega\p{\frac{d}{\eps^2}\log\p{d/\delta} + \log\p{m}}$,
  \noindent then
  \[ \Prob_{\vec{\pi} \sim \distr^n_{\vec{\theta}^{\star}, \pi_0, \matB}} \left( \KL\p{\distr_{\hat{\vec{\phi}}, \hat{\pi}, \matB} || \distr_{\vec{\phi}^{\star}, \pi_0, \matB}} \le \eps^2 \right) \ge 1 - \delta. \]
\end{lemma}

\begin{lemma} \label{lem:learningKLDivergence:blockModel:LowerBound}
    For any $m \in \N$, $d \le m$, there exists a partition $\matB$ of $[m]$ and
  an $\eps_0 > 0$ such that for all $0 < \eps \le \eps_0$ and
  $n = o\p{\frac{d}{\eps^2}}$, it holds that
  \[ \risk_n(\Mallows_d(\matB)) \ge 2 \eps. \]
\end{lemma}

    \bibliography{ref}

  \appendix

  \section{Proofs of Theorem \ref{thm:exponentialFamiliesProperties}, Lemma \ref{lem:partitionFunctionDecomposition} and Fano's Inequality} \label{sec:app:preliminaries}

\begin{prevproof}{Theorem}{thm:exponentialFamiliesProperties}
  For the parts 1., 2. and 3. we refer the reader to
  \citep{Keener11, NielsenG09}. We present here the proof of 4. because it is
  makes the use of the Taylor's Theorem in the last step comparing to the usual
  expression that appears in the literature.
  \begin{align*}
    \KL \left( \distr_{\vec{\eta}} || \distr_{\vec{\eta}'} \right)
    & = \int p_{\vec{\eta}}(\vecx) \ln \frac{p_{\vec{\eta}} (\vecx) }{p_{\vec{\eta}'} (\vecx)} d \mu(\vec{x}) \\
    & = \int p_{\vec{\eta}}(\vecx) \left( (\vec{\eta} - \vec{\eta}')^T \matT(\vecx) + \alpha(\vec{\eta}') - \alpha(\vec{\eta}) \right) d \mu(\vec{x}) \\
    & = (\vec{\eta} - \vec{\eta}')^T \Exp_{\vecx \sim \distr_{\vec{\eta}}} \left[ \matT(\vecx) \right] + \alpha(\vec{\eta}') - \alpha(\vec{\eta}) \\
    & \overset{\eqref{eq:thm:exponentialProperties:1}}{=} - (\vec{\eta}' - \vec{\eta})^T \nabla \alpha(\vec{\eta}) + \alpha(\vec{\eta}') - \alpha(\vec{\eta}) \\
    & = \left( \vec{\eta}' - \vec{\eta} \right)^T \nabla^2 \alpha(\vec{\xi}) \left( \vec{\eta}' - \vec{\eta} \right)
  \end{align*}
  \noindent where the last step follows from the multidimensional Taylor's
  Theorem for some $\vec{\xi} \in L(\vec{\eta}, \vec{\eta}')$.
\end{prevproof}

\begin{prevproof}{Lemma}{lem:partitionFunctionDecomposition}
    We use the simple but profound one-to-one correspondence between every
  permutation $\sigma \in S_m$ and the vector of numbers
  $(V_1(\sigma, \pi), V_2(\sigma, \pi), \dots, V_m(\sigma, \pi))$, where
  $V_j(\sigma, \pi) \in [0, j - 1]$. According to \cite{Knuth1997} this
  correspondence was first proved by Marshall Hall. Let
  $\Omega_m^l = [l] \times [l + 1] \times \cdots \times [m - 1]$. This
  one-to-one correspondence allows as to write the partition function
  $Z(\vec{\phi})$ in the following way
  \begin{align*}
    Z(\vec{\phi}) & = \sum_{\vec{y} \in \Omega_m^0} \prod_{i = 1}^m \phi_j^{y_j} \\
                  & = \sum_{y_1 \in [0]} \phi_1^{y_1} \left( \sum_{\vec{y} \in \Omega_m^1}
                  \prod_{i = 2}^m \phi_j^{y_j} \right) \\
                  & = \left( \sum_{y_1 \in [0]} \phi_1^{y_1} \right) \left( \sum_{\vec{y} \in \Omega_m^1}
                  \prod_{i = 2}^m \phi_j^{y_j} \right)
  \end{align*}
  \noindent continuing this process recursively the lemma follows.
\end{prevproof}

\section{Omitted Proofs from Section \ref{sec:concentration}} \label{sec:app:concentration}

\begin{prevproof}{Theorem}{thm:totalVariationLowerBoundExponentialFamily}
\begin{align*}
  \TV\p{\distr_{\vec{\eta}}, \distr_{\vec{\eta}'}} & = \sum_{\vecx \in \support} \abs{h(\vecx) \exp\p{\vec{\eta}^T \matT(\vecx) - \alpha(\vec{\eta})} - h(\vecx) \exp\p{\vec{\eta'}^{T} \matT(\vecx) - \alpha(\vec{\eta}')}} \\
  & = \sum_{\vecx \in \support} \sign \p{\distr_{\vec{\eta}}(\vecx) - \distr_{\vec{\eta}'}(\vecx)} h(\vecx) \p{ \exp\p{\vec{\eta}^T \matT(\vecx) - \alpha(\vec{\eta})} - \exp\p{\vec{\eta'}^{T} \matT(\vecx) - \alpha(\vec{\eta}')}} \\
  \intertext{now let $b(\vecx) = \sign \p{\distr_{\vec{\eta}}(\vecx) - \distr_{\vec{\eta}'}(\vecx)} h(\vecx)$, for every
  $\vecx \in \support$ we can define the function $g_{\vecx}(\veta) = b(\vecx) \exp\p{\vec{\eta}^T \matT(\vecx) - \alpha(\vec{\eta})}$ and hence}
  \TV\p{\distr_{\vec{\eta}}, \distr_{\vec{\eta}'}} & = \sum_{\vecx \in \support} g_{\vecx}(\veta) - g_{\vecx}(\veta) \\
  \intertext{additionally we define the function $f(\veta) = \sum_{\vecx \in \support} g_{\vecx}(\veta)$ and hence}
    \TV\p{\distr_{\vec{\eta}}, \distr_{\vec{\eta}'}} & = f(\veta) - f(\veta') \\
  \intertext{now from the multidimensional Mean Value Theorem on $f$ there
  exists $\vec{\xi} \in L(\vec{\eta}, \vec{\eta}')$ such that}
  \TV\p{\distr_{\vec{\eta}}, \distr_{\vec{\eta}'}} & = (\veta - \veta')^T \nabla_{\veta} f(\veta) |_{\veta = \vec{\xi}} \\
  & = \sum_{\vecx \in \support} b(\vecx) \p{\vec{\eta} - \vec{\eta}'}^T \p{\nabla_{\vec{\eta}} \left.\p{ \exp\p{\vec{\eta}^T \matT(\vecx) - \alpha(\vec{\eta})}}\right|_{\vec{\eta} = \vec{\xi}}} \\
    & = \sum_{\vecx \in \support} b(\vecx) \p{\vec{\eta} - \vec{\eta}'}^T \p{\matT(\vecx) - \nabla \alpha(\vec{\xi})} \exp\p{\vec{\xi}^T \matT(\vecx) - \alpha(\vec{\xi})} \\
  & \overset{\eqref{eq:thm:exponentialProperties:1}}{=} \sum_{\vecx \in \support} \sign \p{\distr_{\vec{\eta}}(\vecx) - \distr_{\vec{\eta}'}(\vecx)} \p{\vec{\eta} - \vec{\eta}'}^T \p{\matT(x) - \Exp_{\vecy \sim \distr_{\vec{\xi}}}\b{T(\vecy)}} h(\vecx) \exp\p{\vec{\xi} T(\vecx) - \alpha(\vec{\xi})} \\
  & = \Exp_{\vecx \sim \distr_{\vec{\xi}}} \b{\sign \p{\distr_{\vec{\eta}}(\vecx) - \distr_{\vec{\eta}'}(\vecx)} \p{\vec{\eta} - \vec{\eta}'}^T \p{\matT(x) - \Exp_{\vecy \sim \distr_{\vec{\xi}}}\b{T(\vecy)}}}
\end{align*}
\noindent and the lemma follows.
\end{prevproof}

\section{Omitted Proofs of Section \ref{sec:singleParameter}} \label{sec:app:singleParameter}

\begin{prevproof}{Lemma}{lem:permutationLowerBound}
    Our goal is to apply Fano's Inequality (Theorem \ref{thm:FanoInequality}),
  hence we have to define a family of distributions with an upper bound on their
  KL-divergence and a lower bound on their total variation distance.

    We define the permutations $\pi_1$, $\dots$, $\pi_{\ell}$, with
  $\ell = \floor{\frac{m}{2}}$, using the cycle notation of permutations
  \[ \pi_1 = \p{1~~2}, ~~ \pi_2 = \p{3~~4}, ~~ \cdots, ~~ \pi_i = \p{(2i - 1)~~(2i)}, ~~ \cdots, ~~ \pi_\ell = \p{(m - 1)~~m}. \]
  \noindent For all the distributions that we define we use $\phi = 1/2$. Hence
  our family of distribution is the following
  \[ \family = \set{ \distr_{\phi, \pi_1}, \cdots, \distr_{\phi, \pi_{\ell}} }. \]
  \noindent First we compute the an upper bound on the KL-divergence of any pair
  of the above distributions
  \begin{align*}
    \KL \p{\distr_{\phi, \pi_i} || \distr_{\phi, \pi_j}} & = \sum_{\pi \in S_m} \frac{\phi^{d_K(\pi, \pi_i)}}{Z(\phi)} \ln\p{\frac{\frac{\phi^{d_K(\pi, \pi_i)}}{Z(\phi)}}{\frac{\phi^{d_K(\pi, \pi_j)}}{Z(\phi)}}} \\
    & = \sum_{\pi \in S_m} \frac{\phi^{d_K(\pi, \pi_i)}}{Z(\phi)} \ln\p{\phi^{d_K(\pi, \pi_i) - d_K(\pi, \pi_j)}} \\
    & = \ln\p{\phi} \sum_{\pi \in S_m} \frac{\phi^{d_K(\pi, \pi_i)}}{Z(\phi)} \p{d_K(\pi, \pi_i) - d_K(\pi, \pi_j)} \\
    & = \ln\p{1/\phi} \cdot \Exp_{\pi \sim \distr_{\phi, \pi_i}} \b{d_K(\pi, \pi_j) - d_K(\pi, \pi_i)}
  \end{align*}
  \noindent now because of triangle inequality of the Kendall tau distance we
  have that $d_K(\pi, \pi_j) \le d_K(\pi, \pi_i) + d_K(\pi_i, \pi_j)$ and from
  the definition of $\pi_i$, $\pi_j$ we also get that $d_K(\pi_i, \pi_j) = 2$,
  hence $d_K(\pi, \pi_j) - d_K(\pi, \pi_i) \le 2$ and using also that
  $\phi = 1/2$ we have the following bound
  \begin{equation} \label{eq:permutationLowerBound:KLdivergence}
    \KL \p{\distr_{\phi, \pi_i} || \distr_{\phi, \pi_j}} \le 2 \ln\p{2}.
  \end{equation}

    To lower bound the total variation distance between any two distributions
  in $\family$ we use the following claim proved in \citep{LiuM18}.

  \begin{claim}[Claim 1 of \citep{LiuM18}] \label{clm:permutationLowerBound:TVLowerBound}
    For any $\pi, \pi' \in S_m$ with $\pi \neq \pi'$ and any
    $\phi_1, \phi_2 \in [0, 1 - \gamma]$ we have
    \[ \TV\p{\distr_{\phi_1, \pi}, \distr_{\phi_2, \pi'}} \ge \frac{\gamma}{2}. \]
  \end{claim}

  \noindent Therefore from the above claim we immediately get that for any
  $i, j \in [m]$ it holds that
  \begin{equation} \label{eq:permutationLowerBound:TVLowerBound}
    \TV\p{\distr_{\phi, \pi_i}, \distr_{\phi, \pi_j}} \ge \frac{1}{4}.
  \end{equation}

    We can now apply Theorem \ref{thm:FanoInequality} with $\alpha = 1/4$ and
  $\beta = 2 \ln\p{2}$ and we get
  \[ \risk_n(\family) \ge \frac{1}{8} \left(1 - \frac{n \cdot 2 \ln 2 + \ln 2 }{\ln\p{m} - \ln{2}}\right) \]
  from which we get that if $n = o(\log\p{m})$ then
  $\risk_n(\family) \ge \frac{1}{16}$ hence we cannot learn
  $\distr_{\phi, \pi_0}$ $\eps$-close in total variation distance unless
  $n = O(\log(m))$.
\end{prevproof}

\section{Omitted Proof of Section \ref{sec:blockModel}} \label{sec:app:blockModel}

\begin{prevproof}{Lemma}{lem:truncatedGeometricAndMallows}
    The bijective map can be given as
  $h(\sigma)=(V_1 (\sigma, \pi), \dots ,V_m (\sigma, \pi))$. Based on Lemma
  \ref{lem:partitionFunctionDecomposition}, we know that
  $(V_1 (X, \pi), \dots ,V_m (X, \pi))$ are independent random variables if
  $X \sim \Mallows(\pi, \vec{\phi})$, thus their joint distribution can be
  written as in (\ref{eq:vectorOfTruncatedGeometric}) which is equivalent to the
  definition of Generalized Mallows model. The second part of the claim readily
  follows from the existence of the bijective map $h$ that preserves the
  probability mass.
\end{prevproof}

\begin{prevproof}{Theorem}{thm:centralRankingLearningBlockModel}
  The lower bound comes from the lower bound that is given for single parameter Mallows model in Theorem 3.7 of \citep{CaragiannisPS16}. The proof of upper bound for Mallows Block model follows closely the proof of Theorem 3.6 of \citep{CaragiannisPS16}. Let us assume that we are given i.i.d. samples $\pi_i, \dots, \pi_n$ where $n \ge \frac{1}{2\cdot c} \log \frac{m^2}{\delta}$ from $\distr_{\vec{\phi}, \pi_0, \matB}$ with $c=\min_{i,j\in [m]: \pi_0 (i) < \pi_0 (j)} p_{i,j} - p_{j,i}$ where $p_{i,j}$ is the pairwise marginal for item $i$ and $j$ under $\distr_{\vec{\phi}, \pi_0, \matB}$, i.e. $p_{i,j} = \sum_{\pi \in S_m : \pi(i)< \pi(j) } \distr_{\vec{\phi}, \pi_0, \matB}(\pi ) $. Then let us define a ranking $\hat{\pi}$ such that $\hat{\pi}(i) < \hat{\pi}(j) \Longleftrightarrow n_{i,j} > n_{j,i}$ where $n_{i,j}$ is the number of ranking in the sample for which $\pi_i(i) < \pi_i(j)$. Then, using the union bound, we have
  \[
  \Prob \left( d_K (\pi_0 , \hat{\pi} \right) > 0 ) \le \binom{m}{2} 2 e^{-2 c^2 n} \le m^2 e^{-2c^2 n} = \delta
  \]
  What remains is to show that $c$ is constant. For any $i \in B_{i'}$ and $j\in B_{j'}$, it easy to see that $p_{i,j} - p_{j,i} = \Omega(\big(1-\tfrac{\phi_i+\phi_j}{2}\big)\big(1+\tfrac{\phi_i+\phi_j}{2}\big))$ which concludes the proof.
\end{prevproof}

\begin{lemma} \label{lem:learningKLDivergence:blockModel:technicalLemma}
    Let $\expFamily(\matT, h)$ be an exponential family with sufficient
  statistics $\matT$ and carrier measure $h$. For any
  $\distr_{\vec{\eta}} \in \expFamily(\matT, h)$ let $\Domain_{\vec{\eta}}$ be
  the distribution of the corresponding sufficient statistics, i.e.
  $\Domain_{\vec{\eta}}$ is the distribution of $\matT(\vecx)$ when
  $\vecx \sim \distr_{\vec{\eta}}$. Then for all
  $\vec{\eta}, \vec{\eta}' \in \mathcal{H}_{\matT, h}$
  \[ \TV\p{\distr_{\vec{\eta}}, \distr_{\vec{\eta}'}} = \TV\p{\Domain_{\vec{\eta}}, \Domain_{\vec{\eta}'}} ~~ \mathrm{ and } ~~ \KL\p{\distr_{\vec{\eta}} || \distr_{\vec{\eta}'}} = \KL\p{\Domain_{\vec{\eta}} || \Domain_{\vec{\eta}'}}. \]
\end{lemma}

\begin{prevproof}{Lemma}{lem:learningKLDivergence:blockModel:technicalLemma}
    We prove the statement for discrete distributions since this is the
  version of the lemma that we are going to use later in this section but with
  the same arguments we can prove the lemma for continuous distributions too.
  Let $R$ be the support of the exponential family $\expFamily(\matT, h)$,
  $R_{\matT} = \set{\vect \mid \exists \vecx \in R ~:~ \matT(\vecx) = \vect}$ and let
  also
  \[ Q\b{\vec{t}} = \sum_{\vec{x} \in R} \chara \set{\matT(\vec{x}) = \vect}. \]
  We have that
  \begin{align*}
    \KL\p{\distr_{\vec{\eta}} || \distr_{\vec{\eta}'}} & = \sum_{\vec{x} \in R} p_{\vec{\eta}}\p{\vec{x}} \ln \p{\frac{p_{\vec{\eta}}\p{\vecx}}{p_{\vec{\eta}'}\p{\vecx}}} \\
    & = \sum_{\vec{x} \in R} h(\vec{x}) \exp\p{\vec{\eta}^T \matT(\vec{x}) - \alpha\p{\vec{\eta}}} \ln \p{\frac{p_{\vec{\eta}}\p{\vecx}}{p_{\vec{\eta}'}\p{\vecx}}} \\
    & = \sum_{\vect \in R_{\matT}} \p{\sum_{\vecx : \matT\p{\vecx} = \vect} h(\vec{x}) \exp\p{\vec{\eta}^T \matT(\vec{x}) - \alpha\p{\vec{\eta}}}} \ln \p{\frac{p_{\vec{\eta}}\p{\vecx}}{p_{\vec{\eta}'}\p{\vecx}}} \\
    & = \sum_{\vect \in R_{\matT}} \p{Q\b{\vect} h(\vec{x}) \exp\p{\vec{\eta}^T \matT(\vec{x}) - \alpha\p{\vec{\eta}}}} \ln \p{\frac{Q\b{\vect} p_{\vec{\eta}}\p{\vecx}}{Q\b{\vect} p_{\vec{\eta}'}\p{\vecx}}} \\
    & = \sum_{\vect \in R_{\matT}} d_{\vec{\eta}}(\vect) \ln \p{\frac{d_{\vec{\eta}}\p{\vect}}{d_{\vec{\eta}'}\p{\vect}}} = \KL\p{\Domain_{\vec{\eta}} || \Domain_{\vec{\eta}'}}.
  \end{align*}

  \begin{align*}
    \TV\p{\distr_{\vec{\eta}} , \distr_{\vec{\eta}'}} & = \frac{1}{2} \sum_{\vec{x} \in R} \abs{p_{\vec{\eta}}\p{\vec{x}} - p_{\vec{\eta}'}\p{\vecx}} \\
    & = \frac{1}{2} \sum_{\vec{x} \in R} \abs{h(\vec{x}) \exp\p{\vec{\eta}^T \matT(\vec{x}) - \alpha\p{\vec{\eta}}} - h(\vec{x}) \exp\p{\vec{\eta}'^T \matT(\vec{x}) - \alpha\p{\vec{\eta}}}} \\
    & = \frac{1}{2} \sum_{\vect \in R_{\matT}} \sum_{\vecx : \matT\p{\vecx} = \vect} \abs{h(\vec{x}) \exp\p{\vec{\eta}^T \matT(\vec{x}) - \alpha\p{\vec{\eta}}} - h(\vec{x}) \exp\p{\vec{\eta}'^T \matT(\vec{x}) - \alpha\p{\vec{\eta}}}} \\
    & = \frac{1}{2} \sum_{\vect \in R_{\matT}} Q\b{\vect} \abs{h(\vec{x}) \exp\p{\vec{\eta}^T \matT(\vec{x}) - \alpha\p{\vec{\eta}}} - h(\vec{x}) \exp\p{\vec{\eta}'^T \matT(\vec{x}) - \alpha\p{\vec{\eta}}}} \\
    & = \frac{1}{2} \sum_{\vect \in R_{\matT}} \abs{Q\b{\vect} h(\vec{x}) \exp\p{\vec{\eta}^T \matT(\vec{x}) - \alpha\p{\vec{\eta}}} - Q\b{\vect} h(\vec{x}) \exp\p{\vec{\eta}'^T \matT(\vec{x}) - \alpha\p{\vec{\eta}}}} \\
    & = \frac{1}{2} \sum_{\vect \in R_{\matT}} \abs{d_{\vec{\eta}}\p{\vect} - d_{\vec{\eta}'}\p{\vect}} = \TV\p{\Domain_{\vec{\eta}},
    \Domain_{\vec{\eta}'}}.
  \end{align*}
\end{prevproof}

\begin{prevproof}{Lemma}{lem:learningKLDivergence:blockModel:UpperBound}
    First observe that from Theorem \ref{thm:centralRankingLearningBlockModel}
  we can use $O(\log(m/\delta))$ samples to learn the central ranking $\pi_0$.
  Once we know $\pi_0$ we use Lemma \ref{lem:truncatedGeometricAndMallows} and
  hence we can assume that our samples are coming from the distribution
  $\distr_{\vec{\phi}^{\star}, \matB}$ and we want to learn
  $\distr_{\vec{\phi}^{\star}, \matB}$ in KL-divergence. But applying Lemma
  \ref{lem:learningKLDivergence:blockModel:technicalLemma} implies that we can
  assume sample access to the distribution $\Domain_{\vec{\phi}^{\star}, \matB}$
  of the sufficient statistics of $\distr_{\vec{\phi}^{\star}, \matB}$ and we
  want to learn $\Domain_{\vec{\phi}^{\star}, \matB}$ in KL-divergence. From the
  definition of $\distr_{\vec{\phi}^{\star}, \matB}$ we have that the sufficient
  statistics of $\distr_{\vec{\phi}^{\star}, \matB}$ is the vector
  $\matT\p{\vec{z}}$ with $T_i(\vec{z}) = \sum_{j \in B_i} z_j$. Let also
  $\Domain^i_{\phi_i^{\star}, \matB}$ be the distribution of $T_i(\vec{z})$,
  since the coordinates of $\matT\p{\vecx}$ are all independent we get
  \begin{equation} \label{eq:learningKLdivergence:blockModel:KLadditivity}
    \KL\p{\Domain_{\vec{\phi}, \matB} || \Domain_{\vec{\phi}', \matB}} = \sum_{i \in [d]} \KL\p{\Domain^i_{\phi_i, \matB} || \Domain^i_{\phi'_i, \matB}}
  \end{equation}
  \noindent hence it suffices to learn every $\Domain^i_{\phi_i^{\star}, \matB}$
  in KL-divergence with accuracy $\eps/d$ and then we would have learned
  $\distr_{\vec{\phi}^{\star}, \matB}$ in KL-divergence with accuracy $\eps$.

    From the above discussion we have that $\Domain^i_{\phi_i^{\star}, \matB}$ is
  a distribution in an single parameter exponential with natural parameter
  $\theta_i = \ln\p{\phi_i}$, let $\alpha_i$ be the logarithmic partition
  function of the family of $\Domain^i_{\phi_i, \matB}$. From
  \eqref{eq:thm:exponentialProperties:4} we have that
  \[ \KL\p{\Domain^i_{\phi'_i, \matB} || \Domain^i_{\phi_i^{\star}, \matB}} = - \p{\theta'_i - \theta_i^{\star}} \dot{\alpha}_i\p{\theta_i^{\star}} + \alpha_i\p{\theta'_i} - \alpha_i\p{\theta_i^{\star}}. \]
  \noindent We define
  \[ f(x) = - \p{x - \theta_i^{\star}} \dot{\alpha}_i\p{\theta_i^{\star}} + \alpha_i\p{x} - \alpha_i\p{\theta_i^{\star}} \]
  \noindent and we have that
  \[ f'(x) = - \dot{\alpha}_i\p{\theta_i^{\star}} + \dot{\alpha}_i\p{x} \]
  \[ f''(x) = \ddot{\alpha}_i\p{x} \ge 0. \]
  \noindent Hence $f$ is a convex function with minimum value at
  $x = \theta_i^{\star}$. Hence $f$ is a decreasing function for
  $x \le \theta_i^{\star}$ and an increasing function for
  $x \ge \theta_i^{\star}$.

    Observe also that by the definition of $\Domain^i_{\phi_i^{\star}, \matB}$
  and the description of the truncated geometric distribution as discussed in
  Section \ref{sec:model} it holds that
  $\alpha_i(\theta_i) = \sum_{j \in B_i} \ln\p{Z_j(\phi_i)}$. But it is easy to
  see from the definition of $Z_j$ that $Z_j(\phi_i) \ge 1$ and hence
  $\alpha_i(\theta_i) \ge 0$ for all $\theta_i \in (-\infty, 0]$. This
  observation implies $\lim_{x \to - \infty} f(x) = + \infty$ which can be also
  written as
  \begin{equation} \label{eq:learningKLdivergence:blockModel:KLinfity:1}
    \lim_{\phi'_i \to 0} \KL\p{\Domain^i_{\phi'_i, \matB} || \Domain^i_{\phi_i^{\star}, \matB}} = + \infty
  \end{equation}
  \noindent for $\phi_i^{\star} > 0$. The truncated geometric distribution
  $\tgeometric(\phi, k)$ satisfies the symmetry property
  $\tgeometric(1/\phi, k) = k - \tgeometric(\phi, k)$. From this symmetry
  together with \eqref{eq:learningKLdivergence:blockModel:KLinfity:1} we get that
  \begin{equation} \label{eq:learningKLdivergence:blockModel:KLinfity:2}
    \lim_{\phi'_i \to \infty} \KL\p{\Domain^i_{\phi'_i, \matB} || \Domain^i_{\phi_i^{\star}, \matB}} = + \infty
  \end{equation}
  \noindent for $\phi_i^{\star} < + \infty$. We can now define the following set
  \begin{align*}
    Q_i & = \set{\theta \in (-\infty, \infty) \mid \KL\p{\Domain^i_{\phi_i^-, \matB} || \Domain^i_{\phi_i^{\star}, \matB}} \le \eps/d}.
  \end{align*}
  \noindent Because of the convexity of $f$ we know that $Q_i$ is an interval
  such that $\theta_i^{\star} \in Q_i$. From $Q_i$ we can define the following
  parameters
  \begin{align*}
    \theta_i^- = \inf Q_i ~~~ \text{ and } ~~~ \theta_i^+ & = \sup Q_i.
  \end{align*}
  \noindent Observe that because of
  \eqref{eq:learningKLdivergence:blockModel:KLinfity:1} and
  \eqref{eq:learningKLdivergence:blockModel:KLinfity:2} $Q_i$ is a closed
  interval and hence $Q_i = \b{\theta_i^-, \theta_i^+}$ where $\theta_i^-$,
  $\theta_i^+$ are finite numbers not equal to $\pm \infty$. Let
  $\phi_i^- = \theta_i^-$ and $\phi_i^+ = \theta_i^+$. Because of the convexity
  of $f$ and \eqref{eq:learningKLdivergence:blockModel:KLinfity:1},
  \eqref{eq:learningKLdivergence:blockModel:KLinfity:2} we can easily get that
  \begin{equation} \label{eq:learningKLdivergence:blockModel:KLtight:-}
    \KL\p{\Domain^i_{\phi_i^-, \matB} || \Domain^i_{\phi_i^{\star}, \matB}} = \eps/d,
  \end{equation}
  \begin{equation} \label{eq:learningKLdivergence:blockModel:KLtight:+}
    \KL\p{\Domain^i_{\phi_i^+, \matB} || \Domain^i_{\phi_i^{\star}, \matB}} = \eps/d.
  \end{equation}

    Now we apply the same procedure as in the beginning of the proof of Theorem
  \ref{thm:MallowsBlockParameterLearning} and we define the estimator
  $\theta\p{r\p{\vec{\pi}}}$ that satisfies
  \[ \Prob_{\vec{\pi} \sim \distr^n_{\vec{\theta}^{\star}, \pi_0, \matB}} \left( \theta(r(\vec{\pi})) \notin [\theta_i^-, \theta_i^+] \right) \le 2 \exp\left( - \min_{\theta \in \{\theta_i^-, \theta_i^+\}} \KL\left(\distr^i_{\theta} || \distr^i_{\theta_i^{\star}}\right) n \right) \]
  \noindent using \eqref{eq:learningKLdivergence:blockModel:KLtight:-} and \eqref{eq:learningKLdivergence:blockModel:KLtight:+} and the fact that
  $Q_i = \b{\theta_i^-, \theta_i^+}$ this implies
  \[ \Prob_{\vec{\pi} \sim \distr^n_{\vec{\theta}^{\star}, \pi_0, \matB}} \left( \theta(r(\vec{\pi})) \notin Q_i \right) \le 2 \exp\left( - \frac{\eps}{d} n \right). \]
  \noindent Let now $\phi\p{r\p{\vec{\pi}}} = \exp\p{\theta\p{r\p{\vec{\pi}}}}$,
  because of the definition of $Q_i$ we get that
  \[ \Prob_{\vec{\pi} \sim \distr^n_{\vec{\theta}^{\star}, \pi_0, \matB}} \left( \KL\p{\Domain^i_{\phi\p{r\p{\vec{\pi}}}, \matB} || \Domain^i_{\phi_i^{\star}, \matB}} \ge \frac{\eps}{d} \right) \le 2 \exp\left( - \frac{\eps}{d} n \right). \]
  \noindent If we now apply a union bound over all $i \in [d]$ and
  \eqref{eq:learningKLdivergence:blockModel:KLadditivity} we get that
  \[ \Prob_{\vec{\pi} \sim \distr^n_{\vec{\theta}^{\star}, \pi_0, \matB}} \left( \KL\p{\distr_{\vec{\phi}\p{r\p{\vec{\pi}}}, \pi_0, \matB} || \distr_{\vec{\phi}^{\star}, \pi_0, \matB}} \ge \eps \right) \le 2 d \exp\left( - \frac{\eps}{d} n \right). \]
  \noindent Hence for $n \ge \frac{d}{\eps} \ln\p{2d/\delta}$ then
  \[ \Prob_{\vec{\pi} \sim \distr^n_{\vec{\theta}^{\star}, \pi_0, \matB}} \left( \KL\p{\distr_{\vec{\phi}\p{r\p{\vec{\pi}}}, \pi_0, \matB} || \distr_{\vec{\phi}^{\star}, \pi_0, \matB}} \ge \eps \right) \le \delta \]
  \noindent and the lemma follows.
\end{prevproof}

\begin{prevproof}{Lemma}{lem:learningKLDivergence:blockModel:LowerBound}
    Our goal is to apply Fano's Inequality (Theorem \ref{thm:FanoInequality}),
  hence we have to define a family of distributions with an upper bound on their
  KL-divergence and a lower bound on their total variation distance.

    We fix a partition $\matB$ of $[m]$ in equal parts, i.e.
  $\abs{B_i} = m/d$ for all $i \in [d]$. We define the following set of parameters
  $\vec{\phi}$
  \[ \mathcal{G} = \set{\vec{\phi} \mid \phi_i \in \set{\frac{1}{2}, \frac{1}{2} - c \frac{\eps}{\sqrt{m}}}} \]
  \noindent where $c$ is going to be determined later. Based on the Gilbert-Varshamov
  bound we have that there exists a binary code with at least $2^{d/8}$ codewords with
  minimum Hamming distance at least $d/8$. Let $Q$ be such a code, for each codeword
  $q \in Q$ we define vector $\vec{\phi}^{(q)}$ such that
  \[ \phi^{(q)}_i = \left\{ \begin{split}
                              \frac{1}{2} - c \frac{\eps}{\sqrt{m}} & ~~~ \text{if} ~~ q_i = 0 \\
                              \frac{1}{2} ~~~~~~~~~~~~              & ~~~ \text{if} ~~ q_i = 1
                            \end{split} \right.. \]
  \noindent Let $\mathcal{G}' = \set{\vec{\phi}^{(q)} \mid q \in Q}$ and $\pi_0$ be the
  identity permutation, we define the following set of distributions
  \[ \family = \set{\distr_{\vec{\phi}, \pi_0, \matB} \mid \vec{\phi} \in \mathcal{G}'}. \]
  \noindent Because of Lemma \ref{lem:truncatedGeometricAndMallows} we can focus for
  the rest of the proof in the distribution $\distr_{\vec{\phi}, \matB}$. But as we
  have explained the distribution $\distr_{\vec{\phi}, \matB}$ is an $m$ dimensional
  distribution where the $i$th coordinate follows the distribution
  $\tgeometric\p{\phi_i, i - 1}$. If we take any
  $\distr_{\vec{\phi}, \matB},  \distr_{\vec{\phi}', \matB} \in \family$ then by the
  definition of $\family$ we have that
  $\abs{\phi_i - \phi'_i} \le c \frac{\eps}{\sqrt{m}}$ and $\phi_i, \phi'_i \ge 1/4$.
  We can therefore apply \eqref{eq:thm:exponentialProperties:4} and
  \eqref{eq:thm:exponentialProperties:2} to get that for some parameters
  $\psi_i \in [\phi_i, \phi_i'] \cup [\phi'_i, \phi_i]$
  \begin{align*}
    \KL\p{\distr_{\vec{\phi}, \pi_0, \matB} || \distr_{\vec{\phi'}, \pi_0, \matB}} & =
    \sum_{j \in [m]} \p{\ln\p{\phi_j} - \ln\p{\phi'_j}}^2 \Var_{z \sim \tgeometric\p{\psi_j, j - 1}} \b{z} \\
    \intertext{but applying the Lemma \ref{lem:expectationVarianceComputation} and the
    Mean Value Theorem we get that}
    \KL\p{\distr_{\vec{\phi}, \pi_0, \matB} || \distr_{\vec{\phi'}, \pi_0, \matB}} & \le \sum_{j \in [m]} \p{\phi_j - \phi'_j}^2 \frac{1}{\psi_j^2} \frac{\psi_j}{\p{1 - \psi_j}^2}
  \end{align*}
  \noindent but we know that $\psi_j \in [1/4, 1/2]$ and
  $\abs{\phi_i - \phi'_i} \le c \frac{\eps}{\sqrt{m}}$ and therefore
  \begin{align} \label{eq:mainLowerBound:KLUpperBound}
    \KL\p{\distr_{\vec{\phi}, \pi_0, \matB} || \distr_{\vec{\phi'}, \pi_0, \matB}} & \le 32 c^2 \sum_{j \in [m]} \frac{\eps^2}{m} \le 32 c^2 \cdot \eps^2.
  \end{align}

    We now lower bound the total variation distance between any two distributions in
  $\family$. Because of the definition of $\family$ we have that for any
  $\vec{\phi}, \vec{\phi'} \in \mathcal{G}$ they differ in at least $d/8$ coordinates.
  Hence there are at least $d/8$ different $i \in [d]$ such that
  $\phi_i = \frac{1}{2}$ and $\phi'_i = \frac{1}{2} - c \frac{\eps}{\sqrt{m}}$ or
  $\phi_i = \frac{1}{2} - c \frac{\eps}{\sqrt{m}}$ and $\phi'_i = \frac{1}{2}$.
  Therefore for at least $d/16$ of those coordinates we will have that also that
  all $\phi_i$'s are the same and all $\phi'_i$'s are the same. Let $A$ be this set of
  coordinates of $\vec{\phi}$ excluding the coordinates $i \le 4$,
  we define $K = \cup_{a \in A} B_a$ and $k = \abs{K}$ From the definition of $\matB$ we
  have that $k = \frac{d}{16} \frac{m}{d}$. Without loss of generality we assume that
  $\phi \triangleq \phi_i = \frac{1}{2}$ and
  $\phi' \triangleq \phi'_i = \frac{1}{2} - c \frac{\eps}{\sqrt{m}}$. Now we fix
  $\vec{\phi}, \vec{\phi'} \in \mathcal{G}$ and we define $\mathcal{T}_{\vec{\phi}}$ to
  be a copy of the distribution $\distr_{\vec{\phi}, \matB}$ where we keep only the
  coordinates in $K$ and $\mathcal{T}_{\vec{\phi}'}$ to be a copy of the distribution
  $\distr_{\vec{\phi}', \matB}$ where we keep only the coordinates in $K$. Because of
  the definition of $\distr_{\vec{\phi}, \matB}$ we have that
  $\mathcal{T}_{\vec{\phi}}$ is a distribution over vectors $\p{y_1, \dots, y_k}$
  where the all the $y_i$'s are independent and
  $y_i \sim \tgeometric\p{\phi, k_i}$ for some $k_i \in K$. The same way we have that
  $\mathcal{T}_{\vec{\phi}'}$ is a distribution over vectors $\p{y'_1, \dots, y'_k}$
  where the all the $y'_i$'s are independent and
  $y'_i \sim \tgeometric\p{\phi', k_i}$ for some $k_i \in K$.

    From the definition of total variation distance we have that
  \[ \TV\p{\distr_{\vec{\phi}, \matB}, \distr_{\vec{\phi}', \matB}} \ge \TV\p{\mathcal{T}_{\vec{\phi}}, \mathcal{T}_{\vec{\phi}'}}. \]
  \noindent Also we define $T_{\phi}$ to be the distribution of
  $\sum_{i \in [k]} y_i$, where $\p{y_1, \dots, y_k} \sim \mathcal{T}_{\vec{\phi}}$
  and $T_{\phi}'$ to be the distribution of  $\sum_{i \in [k]} y'_i$, where
  $\p{y'_1, \dots, y'_k} \sim \mathcal{T}_{\vec{\phi}'}$. We have that
  \[ \TV\p{\mathcal{T}_{\vec{\phi}}, \mathcal{T}_{\vec{\phi}'}} \ge \TV\p{T_{\phi}, T_{\phi'}} \]
  \noindent and hence
  \[ \TV\p{\distr_{\vec{\phi}, \matB}, \distr_{\vec{\phi}', \matB}} \ge \TV\p{T_{\phi}, T_{\phi'}}. \]

  \noindent It is easy to see now that $T_{\phi}$ is a member of a single
  parameter exponential family with natural parameter $\theta = \ln\p{\phi}$. We
  prove the following claim.

    We now want to apply Theorem
  \ref{thm:totalVariationLowerBoundExponentialFamily} to lower bound the
  quantity $\TV\p{T_{\phi}, T_{\phi'}}$. By the definition of $T_{\phi}$, the
  sufficient statistics of $T_{\phi}$ is $\sum_{i \in [k]} y_i$. Hence let
  $\theta = \ln\p{\phi}$, $\theta' = \ln\p{\phi'}$ and since $\theta > \theta'$
  by the definition of $\phi$, $\phi'$ we have that
  \begin{equation} \label{eq:totalVariationToAbsoluteDeviationPre}
    \TV\p{T_{\phi}, T_{\phi'}} = \Exp_{\vecy
        }\b{\sign\p{T_{\phi}(\vecy) - T_{\phi'}(\vecy)} \p{\sum_{i \in [k]} y_i - \Exp_{\vecz
        }\b{\sum_{i \in [k]} z_i}}} \p{\theta - \theta'}
  \end{equation}
  \noindent where $y_i \sim \tgeometric\p{\psi, k_i - 1}$ and independently
  $z_i \sim \tgeometric\p{\psi, k_i - 1}$. Now from the proof of Theorem
  \ref{thm:totalVariationLowerBoundExponentialFamily} in Section
  \ref{sec:app:concentration}, we have that for every $\vecy$, the sign of
  $\p{\sum_{i \in [k]} y_i - \Exp_{\vecz}\b{\sum_{i \in [k]} z_i}}$ is equal to the
  sign of $\left. \frac{d T_{x}(\vecy)}{d x} \right|_{x = \psi}$. Hence if
  $\left. \frac{d T_{x}(\vecy)}{d x} \right|_{x = \phi} \neq 0$, then from the
  definition of $\phi'$ there exists an $\eps_0 > 0$ such that for every
  $\eps \le \eps_0$ it holds that
  $\left. \frac{d T_{x}(\vecy)}{d x} \right|_{x = \psi}$ does not change sign for
  all $\psi \in [\phi, \phi']$. In this case we have that
  \begin{align} \label{eq:signToAbsolute}
    \sign\p{T_{\phi}(\vecy) - T_{\phi'}(\vecy)} \p{\sum_{i \in [k]} y_i - \Exp_{\vecz}\b{\sum_{i \in [k]} z_i}} = \abs{\sum_{i \in [k]} y_i - \Exp_{\vecz}\b{\sum_{i \in [k]} z_i}}.
  \end{align}
  \noindent To be able to use \ref{eq:signToAbsolute}
  we need to prove that $\left. \frac{d T_{x}(\vecy)}{d x} \right|_{x = \phi} \neq 0$ for every $\vecy$, which is equivalent with
  \[ \sum_{i \in [k]} y_i \neq \Exp_{\vecz}\b{\sum_{i \in [k]} z_i} \]
  where $z_i \sim \tgeometric\p{\phi, k_i - 1}$. We prove this by showing that $\Exp_{\vecz}\b{\sum_{i \in [k]} z_i}$ is not an integer. From Lemma \ref{lem:expectationVarianceComputation} and the fact that $\phi = 1/2$ we have that
  \begin{align} \label{eq:expectationNonInteger}
    \Exp_{\vecz}\b{\sum_{i \in [k]} z_i} = k - \sum_{i \in [k]} k_i \frac{1}{2^{k_i} - 1}
  \end{align}
  \noindent and but the choice of $T_{\phi}$, we have that $k_i \ge 5$ and hence $0 < \sum_{i \in [k]} k_i\frac{1}{2^{k_i} - 1} \le 2 \sum_{i = 5}^{\infty} i \frac{1}{2^i - 1} = 3/4$,
  which implies
  \begin{align} \label{eq:expectationNonIntegerBound}
    \Exp_{\vecz}\b{\sum_{i \in [k]} z_i} \in (k, k + 3/4].
  \end{align}
  \noindent Therefore as we described above it follows that for all $\vecx$, $\vecy$
  it holds that $\left. \frac{d T_{x}(\vecy)}{d x} \right|_{x = \phi} \neq 0$ and hence
  by \eqref{eq:signToAbsolute} we have that
  \begin{equation} \label{eq:totalVariationToAbsoluteDeviation}
    \TV\p{T_{\phi}, T_{\phi'}} = \Exp_{y_i \sim \tgeometric\p{\psi, k_i - 1}}\b{\abs{\sum_{i \in [k]} y_i - \Exp_{z_i \sim \tgeometric\p{\psi, k_i - 1}}\b{\sum_{i \in [k]} z_i}}} \p{\theta - \theta'}
  \end{equation}
  \noindent where $\psi \in [\phi', \phi]$. We now use the following technical claim
  which was first presented in \cite{Tukey1946}.

  \begin{claim}[\citep{Tukey1946}] \label{clm:Tukey}
      For any set $x_1, \dots, x_n$ of independent random variables it holds that
    \[ \Exp\b{\abs{\sum_{i = 1}^n x_i - \sum_{i = 1}^n \Exp\b{x_i}}} \ge \frac{1}{2 \sqrt{2 n}} \sum_{i = 1}^n \Exp\b{\abs{x_i - \Exp\b{x_i}}}. \]
  \end{claim}

  \begin{prevproof}{Claim}{clm:Tukey}
      The inequality as presented in \citep{Tukey1946} holds for random variables with
    zero median, whereas the random variables that we want to use
    $z_i = x_i - \Exp\b{x_i}$ have zero mean. To handle this situation we can use the
    symmetrization argument from the last page of \citep{BirnbaumZ1944}. Tukey's
    inequality together with the symmetrization lemma of \citep{BirnbaumZ1944} give the
    following
    \[ \Exp\b{\abs{\sum_{i = 1}^n x_i - \sum_{i = 1}^n \Exp\b{x_i}}} \ge \frac{1}{2 n} \frac{n!!}{(n - 1)!!} \sum_{i = 1}^n \Exp\b{\abs{x_i - \Exp\b{x_i}}}. \]
    \noindent Now using standard asymptotic formulas of the gamma function we can see
    that
    \[ \frac{n!!}{(n - 1)!!} \ge \sqrt{\frac{n}{2}} \]
    \noindent and the lemma follows.
  \end{prevproof}

    Applying Claim \ref{clm:Tukey} to
  \eqref{eq:totalVariationToAbsoluteDeviation} we get that
  \begin{equation} \label{eq:totalVariationToAbsoluteDeviationAfterTukey}
    \TV\p{T_{\phi}, T_{\phi'}} \ge \p{\theta - \theta'} \frac{1}{2 \sqrt{2 k}} \sum_{i \in [k]} \Exp_{y_i \sim \tgeometric\p{\psi, k_i - 1}}\b{\abs{y_i - \Exp_{z_i \sim \tgeometric\p{\psi, k_i - 1}}\b{z_i}}}.
  \end{equation}
  \noindent Hence it remains to lower bound the absolute deviation of a truncated
  geometric distribution with parameter $\psi \in [\phi', \phi]$. From Lemma
  \ref{lem:expectationVarianceComputation} we have that
  $\Exp_{z_i \sim \tgeometric\p{\psi, k_i - 1}}\b{z_i} \le \frac{\psi}{1 - \psi}$ and
  from the choice of the values of $\phi$, $\phi'$ we have that
  $\frac{\psi}{1 - \psi} \le 1$ hence
  $0 \le \Exp_{z_i \sim \tgeometric\p{\psi, k_i - 1}}\b{z_i} \le 1$. Therefore
  \begin{align*}
    \Exp_{y_i \sim \tgeometric\p{\psi, k_i - 1}} & \b{\abs{y_i - \Exp_{z_i \sim \tgeometric\p{\psi, k_i - 1}}\b{z_i}}} = \sum_{j = 0}^{k_i - 1} \frac{\psi^j}{Z_{k_i - 1}(\psi)} \abs{j - \Exp_{z_i \sim \tgeometric\p{\psi, k_i - 1}}\b{z_i}} \\
    & = \frac{1}{Z_{k_i - 1}(\psi)} \Exp_{z_i \sim \tgeometric\p{\psi, k_i - 1}}\b{z_i} + \sum_{j = 1}^{k_i - 1} \frac{\psi^j}{Z_{k_i - 1}(\psi)} \p{j - \Exp_{z_i \sim \tgeometric\p{\psi, k_i - 1}}\b{z_i}} \\
    & = \frac{2}{Z_{k_i - 1}(\psi)} \Exp_{z_i \sim \tgeometric\p{\psi, k_i - 1}}\b{z_i} + \sum_{j = 0}^{k_i - 1} \frac{\psi^j}{Z_{k_i - 1}(\psi)} \p{j - \Exp_{z_i \sim \tgeometric\p{\psi, k_i - 1}}\b{z_i}} \\
    & = \frac{2}{Z_{k_i - 1}(\psi)} \Exp_{z_i \sim \tgeometric\p{\psi, k_i - 1}}\b{z_i} + \Exp_{z_i \sim \tgeometric\p{\psi, k_i - 1}}\b{z_i} - \Exp_{z_i \sim \tgeometric\p{\psi, k_i - 1}}\b{z_i} \\
    & = \frac{2}{Z_{k_i - 1}(\psi)} \Exp_{z_i \sim \tgeometric\p{\psi, k_i - 1}}\b{z_i} \\
    & \overset{\text{Lemma \ref{lem:expectationVarianceComputation}}}{=} \frac{2 (1 - \psi)}{1 - \psi^{k_i}} \p{\frac{\psi}{1 - \psi} - \frac{k_i \cdot \psi^{k_i}}{1 - \psi^{k_i}}} \\
    & = \frac{2}{\p{1 - \psi^{k_i}}^2} \p{\psi + (k_i - 1) \psi^{k_i + 1} - k_i \psi^{k_i}} \\
    & \ge \frac{1}{\sqrt{2}}
  \end{align*}
  \noindent where for the last inequality we have used the fact that
  $\psi \in [\phi', \phi]$ and the actual values of $\phi', \phi$ together with the fact
  that $k_i \ge 2$. Applying this lower bound to
  \eqref{eq:totalVariationToAbsoluteDeviationAfterTukey} we get that
  \begin{align*}
    \TV\p{T_{\phi}, T_{\phi'}} & \ge \p{\theta - \theta'} \frac{\sqrt{k}}{4} \\
    & = \p{\frac{\ln\p{\phi} - \ln\p{\phi'}}{\phi - \phi'}} \p{\phi - \phi'} \frac{\sqrt{k}}{4}
    \intertext{using Mean Value Theorem and the fact that $\phi = 1/2$, we get that}
    \TV\p{T_{\phi}, T_{\phi'}} & \ge \p{\phi - \phi'} \frac{\sqrt{k}}{2}
    \intertext{but from the definition of $\phi'$ we also have}
    \TV\p{T_{\phi}, T_{\phi'}} & \ge c \frac{\eps}{\sqrt{k}} \frac{\sqrt{k}}{2} = c \frac{\eps}{2}.
  \end{align*}
  \noindent Therefore we get that for any
  $\distr_{\vec{\phi}, \pi_0, \matB}, \distr_{\vec{\phi}', \pi_0, \matB} \in \family$
  it holds
  \begin{equation} \label{eq:TVLowerBoundAfterTukey}
    \TV\p{\distr_{\vec{\phi}, \pi_0, \matB}, \distr_{\vec{\phi}', \pi_0, \matB}} \ge c \frac{\eps}{2}
  \end{equation}

    Using \eqref{eq:mainLowerBound:KLUpperBound} and \eqref{eq:TVLowerBoundAfterTukey},
  we can now apply Theorem \ref{thm:FanoInequality} with $\alpha = c \frac{\eps}{2}$ and
  $\beta = 32 c^2 \eps^2$ and we get
  \[ \risk_n(\family) \ge c \frac{\eps}{4} \left(1 - \frac{n \cdot 32 c^2 \eps^2 + \ln 2 }{\ln\p{\abs{\family}}}\right). \]
  \noindent But from the definition of $\family$ and the Gilbert-Varshamov bound we get
  that $\abs{\family} \ge 2^{d/8}$ and hence
  \[ \risk_n(\family) \ge c \frac{\eps}{4} \left(1 - \frac{n \cdot 32 c^2 \eps^2 + \ln 2 }{d/8}\right). \]
  \noindent Hence we set $c = 8$ and we conclude that for any
  $n \le \frac{d}{2^{14} \eps^2}$ we have
  $\risk_n(\family) \ge 2 \eps$ hence we cannot learn
  $\distr_{\vec{\phi}, \pi_0, \matB}$ $\eps$-close in total variation distance unless
  $n = \Omega\p{\frac{d}{\eps^2}}$.
\end{prevproof}

\section{Proof of Theorem \ref{thm:MallowsBlockParameterLearning}} \label{sec:app:thm:MallowsBlockParameterLearning}

  The estimation $\hat{\pi}$ of $\pi_0$ follows from Theorem
\ref{thm:centralRankingLearningBlockModel}, hence we focus on the estimation
$\hat{\vec{\phi}}$ of $\vec{\phi}^{\star}$. Throughout the proof we assume that
$\matB$ is fixed and hence when drop it from the notation when it is not
necessary. From Lemma \ref{lem:truncatedGeometricAndMallows} and
the expression of the sufficient statistics for the Mallows Block Model we can
conclude that
\begin{equation} \label{eq:prood:MallowsBlock:sufficientStatistics}
  T_i(\pi, \pi_0) = \sum_{j \in B_i} Y_j
\end{equation}
\noindent where $Y_j$ are independent random variables with
$Y_j \sim \tgeometric\p{\phi_i, j - 1}$. Hence we conclude that the random
variables $T_i(\pi, \pi_0)$ are independent and we can estimate them
independently. Therefore we focus in the estimation of each $\phi_i$ separately.
Before continuing we define the distribution $\distr^i_{t}$ to be the
probability distribution of $T_i(\pi, \pi_0)$ where
$\pi \sim \distr_{\vec{\theta}, \pi_0, \matB}$
\footnote{Observe here that we index the distribution with the natural parameter $\theta$ instead of the parameter $\phi$ as we defined it in Section \ref{sec:blockModel}. We may do this indexing in the rest of the proof when it will be clear from the context whether we refer to the natural parameter or the parameter $\phi$.}
with $\theta_i = t$. Also we define
\begin{equation} \label{eq:proof:BlockModelEstimation:separatePartitionFunction}
  Z^i(\phi_i, \matB) = \prod_{j \in B_i} Z_j(\phi_i, \matB)
\end{equation}
\noindent and also
$\alpha_i(\theta_i, \matB) = \ln \p{Z^i\p{\exp\p{\theta_i}, \matB}}$. Again we
may drop the $\matB$ from the notation since it is fixed throughout the proof.

  We fix some $i \in [d]$, and we drop the subscript $i$ from $\theta_i$,
$\phi_i$ since it is clear from the context. We define the function
$h(\theta) = \Exp_{\pi \sim \distr_{\vec{\theta}, \pi_0}} \left[ T_i(\pi, \pi_0) \right]$,
from Theorem \ref{thm:exponentialFamiliesProperties} we have that
$h(\theta) = \dot{\alpha_i}(\theta)$ and also that
$h'(\theta) = \ddot{\alpha}(\theta) > 0$ and hence the function
$h(\theta)$ is strictly increasing with respect to $\theta$. Therefore $h$ is an
injective function and hence given any real number $r$ in the image of $h$ we
can find $\hat{\theta}$ such that
$\abs{\theta(r) - \hat{\theta}} \le \gamma$ in $O(\log(1/\gamma))$ time,
where $\theta(r)$ is well defined from the equation $h(\theta(r)) = r$ since
$h$ is injective.

  Let us assume now that we observe $n$ i.i.d. samples from the distribution
$\distr_{\vec{\theta}^{\star}, \pi_0, \matB} \in \Mallows_r(\matB, \pi_0)$.
Then according to the discussion in the previous paragraph we have that in order
to get an estimation for $\theta_i^{\star}$ is suffices to find a real value
$r(\vec{\pi})$ such that $h(\theta(r(\vec{\pi}))) = r(\vec{\pi})$ and
$\abs{\theta_i^{\star} - \theta(r(\vec{\pi}))} \le \eps$. For this purpose we
are going to use $r = \frac{1}{n} \sum_{i = 1}^n T_i(\pi_i, \pi_0)$. Now from
Theorem \ref{thm:exponentialFamiliesConcentration}, the independence of $T_i$'s
and \ref{eq:prood:MallowsBlock:sufficientStatistics} we have that for any
$\theta_-, \theta_+ \le 0$
\[ \Prob_{\vec{\pi} \sim \distr^n_{\vec{\theta}^{\star}, \pi_0, \matB}} \left( r(\vec{\pi}) \notin [h(\theta_-), h(\theta_+)] \right) \le 2 \exp\left( - \min_{\theta \in \{\theta_-, \theta_+\}} \KL\left(\distr^i_{\theta} || \distr^i_{\theta_i^{\star}}\right) n \right) \]
\noindent Then since $h$ is strictly increasing we have that
\[ \theta(r(\vec{\pi})) \in [\theta_-, \theta_+] \Longleftrightarrow r \in [h(\theta_-), h(\theta_+)] \]
\noindent which together with Theorem \ref{thm:exponentialFamiliesConcentration}
implies
\begin{equation} \label{eq:MallowsBlockConcetration}
  \Prob_{\vec{\pi} \sim \distr^n_{\vec{\theta}^{\star}, \pi_0, \matB}} \left( \theta(r(\vec{\pi})) \notin [\theta_-, \theta_+] \right) \le 2 \exp\left( - \min_{\theta \in \{\theta_-, \theta_+\}} \KL\left(\distr^i_{\theta} || \distr^i_{\theta_i^{\star}}\right) n \right)
\end{equation}

  For the rest of the proof we are going to take two cases that should be
treated a bit differently. The first case is
$\phi_i^{\star} = \exp( \theta_i^{\star} ) > 2 \eps$ and the second case is
$\phi_i^{\star} \le 2 \eps$, where $\eps$ is the accuracy that we want to
estimate the parameter $\phi_i^{\star}$.
\medskip

\paragr{Case $\boldsymbol{\phi_i^{\star} > 2 \eps}$.} Since our goal is to
estimate $\phi_i^{\star} = \exp (\theta_i^{\star})$ we choose
$\theta_- = \log(\phi_i^{\star} - \eps)$ and
$\theta_+ = \log(\phi_i^{\star} + \eps)$. We focus on showing a lower bound
in the KL divergence
$\KL\left(\distr^i_{\theta_-} || \distr^i_{\theta_i^{\star}}\right)$ and a
lower bound on
$\KL\left(\distr^i_{\theta_+} || \distr^i_{\theta_i^{\star}}\right)$ follows
the same way and hence we can apply \eqref{eq:MallowsBlockConcetration}.

 From \eqref{eq:thm:exponentialProperties:4} we have that that for some
$\xi \in [\theta_-, \theta_i^{\star}]$ it holds that
\begin{align*}
  \KL\left(\distr^i_{\theta_-} || \distr^i_{\theta_i^{\star}}\right) & = (\ln(\phi_i^{\star}) - \ln(\phi_i^{\star} - \eps))^2 \ddot{a}_i(\xi) \\
  & = \left( \frac{\log(\phi_i^{\star}) - \log(\phi_i^{\star} - \eps)}{\eps} \right)^2 \eps^2 \ddot{a}_i(\xi) \\
  & = \frac{1}{q^2} \eps^2 \ddot{a}_i(\xi)
\end{align*}
\noindent for some $q \in [\phi_i^{\star} - \eps, \phi_i^{\star}]$ by the Mean
Value Theorem. Hence we have
\begin{equation} \label{eq:proof:blockMallows:KLexpression}
  \KL\left(\distr^i_{\theta_-} || \distr^i_{\theta_i^{\star}}\right) \ge \frac{1}{(\phi_i^{\star})^2} \eps^2 \ddot{\alpha}(\xi),
\end{equation}
\noindent for some $\xi \in [\theta_-, \theta^{\star}]$ and we define
$\psi = \exp(\xi)$. Also from \eqref{eq:thm:exponentialProperties:2} we have
that
\begin{equation} \label{eq:proof:blockMallows:HessianToVariance}
  \ddot{\alpha}_i(\xi) = \Var_{\pi \sim \distr_{\vec{\phi}', \pi_0, \matB}} \left[ T_i(\pi, \pi_0) \right] = \Var_{z \sim \distr^i_{\xi}} \left[ z \right] = .\sum_{j \in B_i} \Var_{Y_j \sim \tgeometric(\xi, j - 1)} \left[ Y_j \right]
\end{equation}
\noindent where $\vec{\phi}'$ is the vector that is equal with
$\vec{\phi}^{\star}$ except that at the $i$th coordinate it has $\xi$. We
therefore need some expressions for the mean and the variance of truncated
geometric distributions. We summarize these expressions in the following Lemma.

\begin{lemma} \label{lem:expectationVarianceComputation}
  Let $k \in \mathbb{N}$, $\phi \in (0, 1)$ then
  \[ \Exp_{Z \sim \tgeometric(\phi, k)} \left[ Z \right] =   \frac{\phi}{1 - \phi} - (k + 1) \frac{\phi^{k + 1}}{1 - \phi^{k + 1}} ~~~~~\text{ and } \]
  \[ \Var_{Z \sim \tgeometric(\phi, k)}\left[ Z \right] = \frac{\phi}{(1-\phi)^2} - \frac{(k+1)^2\phi^{k+1}}{\left(1-\phi^{k+1}\right)^2}. \]
\end{lemma}

\begin{proof}[Proof of Lemma \ref{lem:expectationVarianceComputation}]
  During the proof of this lemma we shall use the fact
  \begin{equation} \label{eq: geometric sum}
  \sum_{\ell = i}^k \phi^\ell
  =
  \phi^i\frac{1-\phi^{k+1-i}}{1-\phi}
  \end{equation}
  at multiple points. In particular,
  \begin{align}
      \Exp_{Z \sim \tgeometric(\phi, k)}\left[ Z \right]
      &=
      \frac{1}{\sum_{j=1}^k\phi^j}
      \sum_{i=1}^{k}i\phi^i
      =
                  \frac{1-\phi}{1-\phi^{k+1}} \sum_{i=1}^{k}i\phi^i
      =
      \frac{1-\phi}{1-\phi^{k+1}}
      \sum_{i=1}^{k}\sum_{j=i}^{k}\phi^j \notag\\
      &=
      \frac{1-\phi}{1-\phi^{k+1}}
      \sum_{i=1}^{k}\phi^i\frac{1-\phi^{k+1-i}}{1-\phi}
      =
      \frac{1}{1-\phi^{k+1}}
      \sum_{i=1}^{k}\left[\phi^i-\phi^{k+1}\right]
      \notag\\
      &=
      \frac{1}{1-\phi^{k+1}}
      \left[\frac{1-\phi^{k+1}}{1-\phi} -1 -k\phi^{k+1}\right]
      =
              \frac{1}{1-\phi^{k+1}}
      \left[\frac{\phi - \phi \phi^{k+1}}{1-\phi} -(k+1)\phi^{k+1}\right]
                          \notag\\
      &=
      \frac{\phi}{1 - \phi} - (k + 1) \frac{\phi^{k + 1}}{1 - \phi^{k + 1}}
      \notag
  \end{align}
  where we have used \eqref{eq: geometric sum} in the second, third and fifth
  step.

  Now we prove compute the variance. Note that
  \begin{equation}
      \label{eq: sum form of squares}
      \sum_{i=1}^j (2i-1)
      = \left(2\sum_{i=1}^j i\right)-j
      = j(j+1) - j
      = j^2 ,
  \end{equation}
  and thus
  \begin{align*}
      \Exp_{Z \sim \tgeometric(\phi, k)}\left[Z^2\right]
      &=
      \sum_{i=1}^{k}i^2\frac{\phi^i}{\sum_{j=1}^k\phi^j}
      \stackrel{\eqref{eq: geometric sum}}{=}
      \frac{1-\phi}{1-\phi^{k+1}}
      \sum_{i=1}^{k}i^2 \phi^i
      \\
      &\stackrel{\eqref{eq: sum form of squares}}{=}
      \frac{1-\phi}{1-\phi^{k+1}}
      \sum_{i=1}^{k}
      \left[(2i-1)\sum_{j=i}^k \phi^j \right]
      \\
      &=
      2\frac{1-\phi}{1-\phi^{k+1}} \sum_{i=1}^{k}
      \left[i\sum_{j=i}^k \phi^j \right]
      -
      \frac{1-\phi}{1-\phi^{k+1}} \sum_{i=1}^{k}
      \sum_{j=i}^k \phi^j
      \\
      &\stackrel{\eqref{eq: geometric sum}}{=}
      2\frac{1-\phi}{1-\phi^{k+1}} \sum_{i=1}^{k}
      \left[i\phi^i \frac{1-\phi^{k+1-i}}{1-\phi} \right]
      -
      \Exp_{Z \sim \tgeometric(\phi, k)}\left[ Z \right]
      \\
      &=
      \frac{2}{1-\phi^{k+1}} \sum_{i=1}^{k}
      \left[i\phi^i - i\phi^{k+1} \right]
      -
      \Exp_{Z \sim \tgeometric(\phi, k)}\left[ Z \right]
      \\
      &=
      \frac{2}{1-\phi}\Exp_{Z \sim \tgeometric(\phi, k)}\left[ Z \right]
      -
      \frac{2\phi^{k+1}}{1-\phi^{k+1}}\frac{k(k+1)}{2}
      -
      \Exp_{Z \sim \tgeometric(\phi, k)}\left[ Z \right]
      \notag\\
      &=
      \frac{1+\phi}{1-\phi}\Exp_{Z \sim \tgeometric(\phi, k)}\left[ Z \right] - \frac{k(k+1)\phi^{k+1}}{1-\phi^{k+1}} .
  \end{align*}
  Consequently,
  \begin{align*}
      \Var_{Z \sim \tgeometric(\phi, k)}\left[ Z \right]
      &=
      \Exp_{Z \sim \tgeometric(\phi, k)}\left[Z^2\right]
      - \p{\Exp_{Z \sim \tgeometric(\phi, k)}\left[Z\right]}^2
      \\&=
      \Exp_{Z \sim \tgeometric(\phi, k)}\left[Z\right]
      \left[
      \frac{1+\phi}{1-\phi}
      -
      \Exp_{Z \sim \tgeometric(\phi, k)}\left[Z\right]
      \right]
       - \frac{k(k+1)\phi^{k+1}}{1-\phi^{k+1}}
       \\&=
      \left[ \frac{\phi}{1-\phi} -
      \frac{(k+1)\phi^{k+1}}{1-\phi^{k+1}} \right]
      \left[ \frac{1}{1-\phi} +
      \frac{(k+1)\phi^{k+1}}{1-\phi^{k+1}} \right]
       - \frac{k(k+1)\phi^{k+1}}{1-\phi^{k+1}}
                              \\
       &=
       \frac{\phi}{(1-\phi)^2}
       -
       \frac{(k+1)\phi^{k+1}}{1-\phi^{k+1}}
       \left[
       (1-\phi)\frac{1}{1-\phi} + (k+1)\frac{\phi^{k+1}}{1-\phi^{k+1}} + k
       \right]
       \\
       &=
       \frac{\phi}{(1-\phi)^2}
       -
       \frac{(k+1)^2\phi^{k+1}}{\left(1-\phi^{k+1}\right)^2} .
  \end{align*}
  \noindent and the lemma follows.
\end{proof}

\noindent Using Lemma \ref{lem:expectationVarianceComputation} and
\eqref{eq:thm:exponentialProperties:2} we get that
\begin{equation} \label{eq:sufficientStatisticVarianceBlockModel}
  \ddot{\alpha}_i(\xi) = \sum_{j \in B_i} \Var_{Y_j \sim \tgeometric(\xi, j - 1)} \left[ Y_j \right] = m_i \frac{\psi}{(1 - \psi)^2} - \sum_{j \in B_i} \frac{j^2 \psi^j}{(1 - \psi^j)^2}
\end{equation}
\noindent where we remind that $m_i = \abs{B_i}$. As we explained already in
order to apply the concentration inequality that we proved in Section
\ref{sec:concentration} we have to lower bound the expression of the variance
and for this we have to prove the following technical claim.

\begin{claim} \label{clm:varianceBoundTechnical}
    Let $x \in [0, 1]$ and $y \in \R_+$ and we define the function
  $g(y) = y^2 \frac{x^y}{\left(1 - x^y\right)^2}$. The function $g$ is an
  decreasing function of $y$.
\end{claim}

\begin{prevproof}{Claim}{clm:varianceBoundTechnical}
    We first compute the derivative of $g$ with respect to $y$ and we get
  \[ g'(y) = \frac{y x^y}{\left(1 - x^y\right)^3} \left(2\left(1 - x^y\right) + y \ln(x) +  y \ln(x)x^y \right). \]
  \noindent The sign of $g'(y)$ is therefore determined by the sign of the
  following quantity
  \[ h(z) = 2(1 - z) + \ln(z) + z \ln(z) \]
  \noindent where we have replaced $z = x^y$ and the only restriction that we
  have is $z \in [0, 1]$. If we compute the derivative of $h$ we have
  \[ h'(z) = - 1 + \frac{1}{z} + \ln(z). \]
  \noindent But we know that $\ln(x) \le x - 1$ and hence $\ln(1/z) \le 1/z - 1$
  which implies $h'(z) \ge 0$. Since $z \in [0, 1]$ we get that $h(z) \le h(1)$
  but $h(1) = 0$ and hence $h(z) \le 0$. From this we get $g'(y) \le 0$ and
  therefore $g$ is a decreasing function of $y$.
\end{prevproof}

  From Claim \ref{clm:varianceBoundTechnical} we get that
$\frac{i^2 \psi^i}{(1 - \psi^i)^2} \ge \frac{(i + 1)^2 \psi^{i + 1}}{(1 - \psi^{i + 1})^2}$
and therefore we the following lower bound in the variance of the sufficient
statistics
\[ \ddot{\alpha}_i(\xi) \ge m_i \frac{\psi}{(1 - \psi)^2} - m_i \frac{4 \psi^2}{(1 - \psi^2)^2} \]
\noindent where we have replaced all the terms in the sum in the expression
\eqref{eq:sufficientStatisticVarianceBlockModel} with $i \ge 2$ with $i = 2$.
The case $i = 1$ corresponds to a trivial delta distribution that does not
contribute in any part of the proof of this section. Since
$\psi \in [\phi_i^{\star} - \eps, \phi_i^{\star}]$ we get that
\begin{align*}
  \ddot{\alpha}_i(\xi) & \ge m_i \frac{\psi}{(1 - \psi^2)^2} \left( (1 + \psi)^2 - 4 \psi \right) \\
  & = m_i \frac{\psi}{(1 + \psi)^2} \ge \frac{1}{4} m_i \psi \implies \\
  \ddot{\alpha}_i(\xi) & \ge \frac{1}{4} m_i (\phi_i^{\star} - \eps).
\end{align*}

  Now we use \eqref{eq:proof:blockMallows:KLexpression} and
\eqref{eq:proof:blockMallows:HessianToVariance} together with the above lower
bound and we get
\[ \KL\left(\distr^i_{\theta_-} || \distr^i_{\theta_i^{\star}}\right) \ge \frac{1}{4} (m - 1) \eps^2 \psi \]
where $\psi \in \left[\phi_i^{\star} - \eps, \phi_i^{\star}\right]$. Using
exactly the same argument we can also prove the same for
$\distr^i_{\theta_+}$ and
$\psi \in \left[\phi_i^{\star}, \phi_i^{\star} + \eps\right]$ and therefore we
get
\[ \min_{\theta \in \{\theta_-, \theta_+\}} \KL\left(\distr^i_{\theta}
|| \distr^i_{\theta_i^{\star}}\right) \ge \frac{1}{4} m_i \eps^2 \min \left\{ \frac{\phi_i^{\star} - \eps}{(\phi_i^{\star})^2}, \frac{\phi_i^{\star}}{(\phi_i^{\star} + \eps)^2} \right\}.
\]
\noindent Since the function $x \mapsto \frac{x - \eps}{x^2}$ and the function
$x \mapsto \frac{x}{(x + \eps)^2}$ are decreasing functions of $x$ for
$x \in [2 \eps, 1]$ and assuming that $\eps \le 3/4$ we have that
\begin{equation} \label{eq:proof:blockMallows:KLFinalBound:1}
  \min_{\theta \in \{\theta_-, \theta_+\}} \KL\left(\distr^i_{\theta}
|| \distr^i_{\theta_i^{\star}}\right) \ge \frac{1}{16} m_i \eps^2.
\end{equation}

  We can now apply \eqref{eq:proof:blockMallows:KLFinalBound:1} to
\eqref{eq:MallowsBlockConcetration} and we get
\[ \Prob_{\vec{\pi} \sim \distr^n_{\vec{\theta}^{\star}, \pi_0, \matB}} \left( \theta(r(\vec{\pi})) \notin [\theta_-, \theta_+] \right) \le 2 \exp\left( - \frac{1}{16} m_i \eps^2 n \right). \]
\noindent Hence for $n \ge 16 \frac{\ln(2/\delta)}{m_i \eps^2}$ we have that
\[ \Prob_{\vec{\pi} \sim \distr^n_{\vec{\theta}^{\star}, \pi_0, \matB}} \left( \theta(r(\vec{\pi})) \notin [\theta_-, \theta_+] \right) \le \delta. \]

\paragr{Case $\boldsymbol{\phi_i^{\star} \le 2 \eps}$.} For this case we will
set $\theta_- = 0$ and $\theta_+ = \phi_i^{\star} + k \eps$, where $k \in \N$ to
be determined later. Hence, from \eqref{eq:thm:exponentialProperties:4} we have
that
\begin{align*}
  \KL\left(\distr^i_{\theta_+} || \distr^i_{\theta_i^{\star}}\right) & = (\ln(\phi_i^{\star} + k \eps) - \ln(\phi_i^{\star})) \cdot \dot{\alpha}_i\left(\ln \left(\phi_i^{\star} + k \eps\right)\right) + \alpha\left(\ln\left(\phi_i^{\star}\right)\right) - \alpha\left(\ln\left(\phi_i^{\star} + k \eps\right)\right)
\end{align*}

  Our first goal is to show that for $\phi_i^{\star} \le 2 \eps$ the right hand
side of the KL-divergence is a decreasing function of $\phi_i^{\star}$. We set
\[ f(x) \triangleq \p{\ln\p{x + k \eps} - \ln\p{x}} \cdot \dot{\alpha}_i\p{\ln\p{x + k \eps}} + \alpha_i\p{\ln\p{x}} - \alpha_i\p{\ln\p{x + k \eps}} \]
\noindent we get that
\begin{align*}
  f'(x) & = \p{\frac{1}{x + k \eps} - \frac{1}{x}} \cdot \dot{\alpha}_i\p{\ln\p{x + k \eps}} + \p{\ln\p{x + k \eps} - \ln\p{x}} \cdot \frac{\ddot{\alpha}_i\p{\ln\p{x + k \eps}}}{x + k \eps} + \\
  & ~~~~~~~~~~ + \frac{\dot{\alpha}_i\p{\ln\p{x}}}{x} - \frac{\dot{\alpha}_i\p{\ln\p{x + k \eps}}}{x + k \eps} \\
  & = \frac{\dot{\alpha}_i\p{\ln(x)} - \dot{\alpha}_i\p{\ln{x + k \eps}}}{x} + \p{\ln\p{x + k \eps} - \ln\p{x}} \cdot \frac{\ddot{\alpha}_i\p{\ln\p{x + k \eps}}}{x + k \eps} \\
  & =  -\frac{\ddot{\alpha}_i\p{\ln\p{x + k \eps}}}{x} \left(\frac{\dot{\alpha}_i\p{\ln\p{x + k \eps}} - \dot{\alpha}_i\p{\ln\p{x}}}{\ddot{\alpha}_i\p{\ln\p{x + k \eps}}} + \frac{x}{x + k \eps}\p{\ln\p{\frac{x}{x + k \eps}}}\right).
  \intertext{We use now the easy to check facts that (1) the function $z \mapsto z \ln\p{z}$ is a decreasing function of $z$ for $z \le 1/e$, (2) the function $x \mapsto \frac{x}{x + k \eps}$ is an increasing function of $x$, (3) we pick $k$ such that for $x \in [0, 2 \eps]$ we have that $\frac{x}{x + k \eps} \le \frac{1}{e}$ and hence we get that}
  f'(x) & \le -\frac{\ddot{\alpha}_i\p{\ln\p{x + k \eps}}}{x} \left(\frac{\dot{\alpha}_i\p{\ln\p{x + k \eps}} - \dot{\alpha}_i\p{\ln\p{x}}}{\ddot{\alpha}_i\p{\ln\p{x + k \eps}}} - \frac{2}{(k + 2)}\p{\ln\p{\frac{k + 2}{2}}}\right).
\end{align*}

  Now we want to lower bound the term
$\frac{\dot{\alpha}_i\p{\ln\p{x + k \eps}} - \dot{\alpha}_i\p{\ln\p{x}}}{\ddot{\alpha}_i\p{\ln\p{x + k \eps}}}$
in the parentheses in the last upper bound of $f'(x)$. From
\eqref{eq:thm:exponentialProperties:1} and
\eqref{eq:thm:exponentialProperties:2} we have that
\begin{align*}
  \frac{\dot{\alpha}_i\p{\ln\p{x + k \eps}} - \dot{\alpha}_i\p{\ln\p{x}}}{\ddot{\alpha}_i\p{\ln\p{x + k \eps}}} & = \frac{\Exp_{z \sim \distr^i_{\ln\p{x + k \eps}}}\left[z\right] - \Exp_{z \sim \distr^i_{\ln\p{x}}}\left[z\right]}{\Var_{z \sim \distr^i_{\ln\p{x + k \eps}}}\left[z\right]} \triangleq \frac{E}{D}
\end{align*}

\noindent To lower bound this expression we use the following simple claim.

\begin{claim} \label{clm:expectationBoundTechnical}
    Let $x \in [0, 1]$ and $y \in \R_+$ and we define the function
  $g(y) = y \frac{x^y}{\left(1 - x^y\right)}$. The function $g$ is an
  decreasing function of $y$.
\end{claim}

\begin{prevproof}{Claim}{clm:varianceBoundTechnical}
    We first compute the derivative of $g$ with respect to $y$ and we get
  \[ g'(y) = \frac{x^y}{\left(1 - x^y\right)^2} \left(\left(1 - x^y\right) + y \ln(x) \right). \]
  \noindent The sign of $g'(y)$ is therefore determined by the sign of the
  following quantity
  \[ h(z) = (1 - z) + \ln(z) \]
  \noindent where we have replaced $z = x^y$ and the only restriction that we
  have is $z \in [0, 1]$. But we know that $\ln(x) \le x - 1$ and hence
  $h(z) \le 0$ which implies $g'(y) \le 0$ and the claim follows.
\end{prevproof}

\noindent From Lemma \ref{lem:expectationVarianceComputation} and Claim
\ref{clm:expectationBoundTechnical} we have that
\begin{align*}
  \Exp_{z \sim \distr^i_{\ln\p{x + k \eps}}}\left[z\right] & = \frac{m_i(x + k \eps)}{(1 - (x + k \eps))} - \sum_{j \in B_i} \frac{j (x + k \eps)^j}{\p{1 - (x + k \eps)^j}} \\
  & \ge \frac{m_i (x + k \eps)}{(1 - (x + k \eps))} - \frac{2 m_i (x + k \eps)^2}{\p{1 - (x + k \eps)^2}} \\
  & = \frac{m_i(x + k \eps)}{(1 + x + k \eps)}
\end{align*}
\noindent where again we have excluded the trivial case $j = 1$ that does not
contribute to the above expression. It is also direct from Lemma
\ref{lem:expectationVarianceComputation} that
\[\Exp_{z \sim \distr^i_{\ln\p{x}}}\left[z\right] \le \frac{m_i x}{(1 - x)} \]

\noindent From these two bounds, the fact that $x \in [0, 2 \eps]$ and the
assuming that $\eps \le \frac{1}{10 k}$ we conclude that
\begin{align*}
  E & \ge \frac{m_i k \eps}{1 + k \eps} - \frac{m_i 2 \eps}{1 -  2 \eps} = m_i \eps \left( \frac{k}{1 + k \eps} - \frac{2}{1 - 2 \eps} \right) \ge m_i \eps \frac{k - 2}{(1 + k \eps)} \\
  & \ge m_i \eps (k - 2) \frac{10}{11}
\end{align*}

\noindent Also directly from  Lemma \ref{lem:expectationVarianceComputation},
the fact that $x \in [0, 2 \eps]$ and assuming $\eps \le \frac{1}{10(k + 2)}$ we
get that
\[ D = \Var_{z \sim \distr^i_{\ln\p{x + k \eps}}}\left[z\right] \le \frac{m_i (k + 2) \eps}{(1 - (k + 2) \eps)^2} \le m_i (k + 2) \eps \frac{100}{81} \]

\noindent Putting all these together we get
\[ \frac{\dot{\alpha}_i\p{\ln\p{x + k \eps}} - \dot{\alpha}_i\p{\ln\p{x}}}{\ddot{\alpha}_i\p{\ln\p{x + k \eps}}} \ge \frac{81}{110} \frac{k - 2}{k + 2} \]

\noindent Hence we have the following upper bound on $f'(x)$
\begin{align*}
  f'(x) & \le -\frac{\ddot{\alpha}_1\p{\ln\p{x + k \eps}}}{x} \left(\frac{81}{110}\frac{k - 2}{k + 2} - \frac{2}{(k + 2)}\p{\ln\p{\frac{k + 2}{2}}}\right) \\
  \intertext{for $k = 14$ we have that}
  f'(x) & \le -\frac{\ddot{\alpha}_1\p{\ln\p{x + k \eps}}}{x} \left(\frac{243}{440} - \frac{1}{8}\ln\p{8}\right) \le 0 \\
\end{align*}

\noindent Therefore we have that for $k = 14$ and $\eps \le \frac{1}{10 k}$
the function $f$ is a decreasing function of $x$ and hence $f(x) \ge f(2 \eps)$
for $x \in [0, 2\eps]$. Let $\theta' = \exp\p{\ln\p{2 \eps}}$ and
$\theta'_+ = \exp\p{\ln\p{(k + 2) \eps}}$ then we have
\begin{equation} \label{eq:blockMallows:KLmonotonicity}
  \KL\left(\distr^i_{\theta_+} || \distr^i_{\theta_i^{\star}}\right) \ge \KL\left(\distr^i_{\theta'_+} || \distr_{\theta'_i}\right).
\end{equation}
\noindent Hence we can now use \eqref{eq:proof:blockMallows:KLFinalBound:1} to
bound the right hand side and we get
\begin{equation} \label{eq:blockMallows:KLlowerBoundLowVariance}
  \KL\left(\distr^i_{\theta_+} || \distr_{\theta_i^{\star}}\right) \ge \frac{(k + 2)^2}{16} m_i \eps^2.
\end{equation}
\noindent and now we can use \eqref{eq:MallowsBlockConcetration} to get that
for $n \ge \frac{16}{(k + 2)^2} \frac{\ln(2/\delta)}{m_i \eps^2}$ and
$\phi^{\star} \in [0, 2 \eps]$ it holds
\[ \Prob_{\vec{\pi} \sim \distr^n_{\vec{\theta}^{\star}, \pi_0, \matB}} \left( \theta(r) \notin [0, \theta_2] \right) \le \delta. \]

\noindent Now if we combine the results that we have for the two regimes
$\phi_i^{\star} \ge 2 \eps$, $\phi_i^{\star} < 2 \eps$ and given that
we computing $\hat{\phi}_i$ such that
$\abs{\hat{\phi}_i - \exp\left(\theta_i(r(\vec{\pi}))\right)} \le \eps$ we get
that for any $n \ge \frac{1}{16} \frac{\ln(2/\delta)}{m_i \eps^2}$ it holds that
\[ \Prob_{\vec{\pi} \sim \distr^n_{\vec{\theta}^{\star}, \pi_0, \matB}} \left(\hat{\phi}_i \notin \left[\phi_i^{\star} - \eps, \phi_i^{\star} + \eps\right] \right) \le \delta. \]
\noindent for any $i \in [d]$. Our goal of course is to compute an estimate
$\hat{\vec{\phi}}$ such that the total $\ell_2$ error from all coordinates is
less than $\eps$. To do so we estimate each $\phi_i^{\star}$ with accuracy
$\eps' = \eps/\sqrt{d}$ and with error probability $\delta' = \delta/d$.
Therefore we have that for any
$n \ge \frac{1}{16} \frac{d}{m_i \eps^2} \ln\p{d/\delta}$ is holds that
\[ \Prob_{\vec{\pi} \sim \distr^n_{\vec{\theta}^{\star}, \pi_0, \matB}} \left( \abs{\hat{\phi}_i - \phi_i^{\star}} \ge \frac{\eps}{\sqrt{d}} \right) \le \frac{\delta}{d} \]
\noindent and therefore using union bound over all coordinates we get that
\[ \Prob_{\vec{\pi} \sim \distr^n_{\vec{\theta}^{\star}, \pi_0, \matB}} \left( \norm{\hat{\vec{\phi}} - \vec{\phi}^{\star}}_2 \ge \eps \right) \le \delta \]
\noindent and Theorem \ref{thm:MallowsBlockParameterLearning} follows.


\end{document}